\documentclass[oneside,11pt]{article}

\usepackage{times,url,mathrsfs}
\usepackage{amsmath,wrapfig,color,bigfoot,enumerate}
\usepackage{amsfonts}
\usepackage{graphicx}
\usepackage{subfigure}
\newtheorem{theorem}{Theorem}

\newtheorem{proof}{Proof}
\newtheorem{corollary}{Corollary}

\newtheorem{mydef}{Definition}
\newtheorem{fact}{Fact}
\begin{document}

\title{\Huge Asymmetric Minwise Hashing}

\author{
\textbf{\Large Anshumali Shrivastava}\vspace{0.05in}\\
       {Department of Computer Science}\\
       {Computer and Information Science}\\
       {Cornell University}\\
      {Ithaca, NY 14853, USA}\\
       \texttt{anshu@cs.cornell.edu}
\and
\textbf{\Large Ping Li} \vspace{0.05in}\\
         Department of Statistics and Biostatistics\\
         Department of Computer Science\\
       Rutgers University\\
          Piscataway, NJ 08854, USA\\
       \texttt{pingli@stat.rutgers.edu}
}

\date{}
\maketitle

\begin{abstract}
\noindent Minwise hashing (Minhash) is a widely popular indexing scheme in practice. Minhash is designed for estimating set resemblance and is known to be suboptimal in many applications  where the desired measure is set overlap (i.e., inner product between binary vectors) or set containment. Minhash has inherent bias towards smaller sets, which adversely affects its performance in applications where such a penalization is not desirable.  In this paper, we propose asymmetric minwise hashing ({\em MH-ALSH}), to  provide a solution to this problem. The new scheme utilizes asymmetric transformations to cancel the bias of traditional minhash towards smaller sets, making the final ``collision probability'' monotonic in the inner product.  Our theoretical comparisons show that for the task of retrieving with binary inner products asymmetric minhash is provably better than traditional minhash and other recently proposed hashing algorithms for general inner products. Thus, we obtain an algorithmic improvement over existing approaches in the literature. Experimental evaluations on four publicly available high-dimensional datasets validate our claims and the proposed scheme outperforms, often significantly, other hashing algorithms on the task of near neighbor retrieval with set containment. Our proposal is simple and easy to implement in practice.
\end{abstract}

\newpage\clearpage

\section{Introduction}\label{sec:intro}

Record matching (or linkage), data cleansing and plagiarism detection are among the most frequent operations in  many large-scale data processing systems over the web.  \emph{Minwise hashing} (or minhash)~\cite{Proc:Broder,Proc:Broder_STOC98} is a popular technique deployed by big data industries for these tasks. Minhash was originally developed for economically estimating the {\em resemblance} similarity between sets (which can be equivalently viewed as binary vectors). Later, because of its locality sensitive property~\cite{Proc:Indyk_STOC98}, minhash became a widely used hash function for creating hash buckets leading to efficient algorithms for numerous  applications including spam detection~\cite{Proc:Broder}, collaborative filtering~\cite{Proc:Bachrach_IJCAI09}, news personalization~\cite{Proc:Das_WWW07}, compressing social networks~\cite{Proc:Chierichetti_KDD09}, graph sampling~\cite{Proc:Cormode_SIGMOD05}, record linkage~\cite{Proc:Koudas_SIGMOD06}, duplicate detection~\cite{Proc:Henzinger_SIGIR06}, all pair similarity~\cite{Proc:Bayardo_WWW07}, etc.

\subsection{Sparse Binary Data, Set Resemblance, and Set Containment }

Binary representations for web documents are common, largely due to the wide adoption of the ``Bag of Words'' (BoW) representations for  documents and images. In BoW  representations, the word frequencies within a document follow power law.   A significant number of words (or combinations of words) occur rarely in a document  and most of the higher order shingles in the document occur only once.  It is often the case that just the presence or absence information suffices in  practice~\cite{Article:Chapelle_99,Proc:Hein_AISTATS05,Proc:Jiang_CIVR07,Proc:HashLearning_NIPS11}.  Leading search companies routinely use  sparse binary representations in their large data systems~\cite{Report:chandra_10}.

The underlying similarity measure of interest with minhash is the resemblance, also known as the Jaccard similarity. The resemblance similarity between two sets $x$, $y
\subseteq \Omega = \{1,2,...,D\}$ is
\begin{align}
&\mathcal{R} = \frac{|x \cap y|}{| x \cup y|} = \frac{a}{f_x+f_y -a},\\\notag
\text{where } \ & f_x = |x|, \ \ f_y = |y|,\ \ a = |x \cap y|.
\end{align}
 Sets can  be equivalently viewed as binary vectors with each component indicating the presence or absence of an attribute. The cardinality (e.g., $f_x$, $f_y$) is the number of nonzeros in the binary vector.\\

While the resemblance similarity is convenient and useful in numerous applications, there are also many scenarios where the resemblance is not the desirable similarity measure~\cite{Proc:Agrawal_SIGMOD2010,Proc:chaudhuri_ICDE06}.  For instance, consider  text descriptions of two restaurants:
\begin{enumerate}[(i)]
\item ``Five Guys Burgers and Fries Brooklyn New York"
\item ``Five Kitchen Berkley"
 \end{enumerate}
Shingle based representations for strings are common in practice. Typical (first-order) shingle based representations of these names will be (i) \{five, guys, burgers, and, fries, brooklyn, new, york \} and (ii) \{five, kitchen, berkley\}. Now suppose  the query is ``Five Guys" which in shingle representation is \{Five, Guys\}. We would like to match and search the records, for this query ``Five Guys", based on resemblance. Observe that the resemblance between query and record (i) is $\frac{2}{8}$ = 0.25, while that with record (ii) is $\frac{1}{3}$ = 0.33. Thus, simply based on resemblance, record (ii) is a better match for query ``Five Guys" than record (i), which should  not be correct in this content.

Clearly the issue here is that the resemblance penalizes the sizes of the sets involved. Shorter sets are unnecessarily favored over longer ones, which hurts the performance in record matching~\cite{Proc:Agrawal_SIGMOD2010} and other applications. There are many other scenarios where such penalization is undesirable. For instance, in plagiarism detection, it is typically immaterial whether the text is plagiarized from a big or a small document.

To counter the often unnecessary penalization of the sizes of the sets with resemblance, a modified measure, the {\em set containment} (or Jaccard containment) was adopted~\cite{Proc:Broder,Proc:Agrawal_SIGMOD2010,Proc:chaudhuri_ICDE06}. Jaccard containment of set $x$ and $y$ with respect to $x$ is defined as
\begin{equation}
\mathcal{J_C} = \frac{|x \cap y|}{| x|} = \frac{a}{f_x}.
\end{equation}
In the above example with query ``Five Guys" the Jaccard containment with respect to query for record (i) will be $\frac{2}{2} = 1$ and
with respect to record (ii) it will be $\frac{1}{2}$, leading to the desired ordering.  It should be noted that for any fixed query $x$, the ordering under Jaccard containment with respect to the query, is the same as the ordering with respect to the intersection $a$ (or binary inner product). Thus, near neighbor search problem with respect to  $\mathcal{J_C}$ is equivalent to the near neighbor search problem with respect to $a$.

\subsection{Maximum Inner Product Search (MIPS) and  Maximum Containment Search (MCS) }

Formally, we state our problem of interest. We are given a collection $\mathcal{C}$ containing  $n$ sets (or binary vectors) over universe $\Omega$ with $|\Omega| = D$ (or binary vectors in $\{0,1\}^D$).  Given a query $q \subset \Omega$, we are interested in the problem of finding $x \in \mathcal{C}$  such that
\begin{align}
 x =  \arg \max_{x \in \mathcal{C}} |x \cap q| = \arg \max_{x \in \mathcal{C}} q^Tx;
\end{align}
where $|\ \ |$ is the cardinality of the set. This is the so-called {\em maximum inner product search (MIPS)} problem.

For binary data, the MIPS  problem is equivalent to  searching with Jaccard containment with respect to the query, because the cardinality of the query does not affect the ordering and hence the $\arg\max$.
\begin{align}
\label{eq:prob}
 x =   \arg \max_{x \in \mathcal{C}} |x \cap q|  = \arg \max_{x \in \mathcal{C}} \frac{|x \cap q| }{ |q|};
\end{align}
which is also referred to as the {\em maximum containment search (MCS)} problem.

\subsection{ Shortcomings of Inverted Index Based Approaches for MIPS (and MCS)}

Owing to its practical  significance,    there have been many existing heuristics for solving the MIPS (or MCS) problem~\cite{Proc:Melnik_TODC03,ramasamy2000set,Report:Chaudhuri_VLDB09}. A notable  recent work among them made use
of the inverted index based approach~\cite{Proc:Agrawal_SIGMOD2010}. Inverted indexes might be
suitable for problems when the sizes of documents are small and each record
only contains few words. This situation, however, is
not commonly observed in practice. The documents over the web are large with
huge vocabulary. Moreover, the vocabulary blows up very quickly once
we start using higher-order shingles. In addition, there is an increasing
interest in enriching the text with extra synonyms to make the search
more effective and robust to semantic meanings~\cite{Proc:Agrawal_SIGMOD2010}, at the cost of a significant increase of   the sizes of the documents. Furthermore, if the query contains many words then the inverted index is not very useful. To mitigate this
issue several additional heuristics were proposed, for instance, the
heuristic based on minimal infrequent sets~\cite{Proc:Agrawal_SIGMOD2010}. Computing minimal
infrequent sets is similar to the set cover problem which is hard in
general and thus~\cite{Proc:Agrawal_SIGMOD2010} resorted to greedy heuristics. The number of minimal infrequent sets could be huge in
general and so these heuristics can be very costly. Also, such heuristics require the
knowledge of the entire dataset before hand which is usually not
practical in a dynamic environment like the web. In addition, inverted
index based approaches do not have  theoretical guarantees on the
query time and their performance is very much dataset dependent.

\subsection{Probabilistic Hashing}

Locality Sensitive Hashing (LSH)~\cite{Proc:Indyk_STOC98} based randomized techniques are common and successful in industrial practice for efficiently solving NNS ({\em near neighbor search}).  They are some of the few known techniques that do not suffer from the curse of dimensionality. Hashing based indexing schemes provide provably sub-linear algorithms for search which is a boon in this era of big data where even linear search algorithms are impractical due to latency.   Furthermore, hashing based indexing schemes are massively parallelizable and can be updated incrementally (on data streams), which  makes them ideal for modern distributed systems. The prime focus of this paper will be on efficient hashing based algorithms for binary inner products.

Despite the interest in Jaccard containment and binary inner products, there were no hashing algorithms for these measures for a long time and minwise hashing is  still a widely popular heuristic~\cite{Proc:Agrawal_SIGMOD2010}.  Very recently, it was shown that general inner products for real vectors can be efficiently solved by using asymmetric locality sensitive hashing schemes~\cite{Proc:Shrivastava_NIPS14,Article:Shrivastava_arXiv14}. The asymmetry is necessary for the general inner products and an impossibility of having a symmetric hash function can be easily shown using elementary arguments.  Thus, binary inner product (or set intersection) being a special case of general inner products also admits provable efficient search algorithms with these asymmetric hash functions which are based on random projections. However, it is known that random projections are suboptimal for retrieval in the sparse binary domain~\cite{Proc:Shrivastava_AISTATS14}.  Hence, it is expected that the existing asymmetric locality sensitive hashing schemes for general inner products are likely to be suboptimal for retrieving with sparse high dimensional binary-like datasets, which are common over the web.

\subsection{Our Contributions}

We investigate hashing based indexing schemes for the problem of near neighbor search with binary inner products and Jaccard containment. Binary inner products are special. The impossibility of existence of LSH for general inner products shown in~\cite{Proc:Shrivastava_NIPS14} does not hold for the binary case. On the contrary, we provide an explicit construction of a provable LSH based on sampling, although our immediate investigation reveals that such an existential result is only good in theory and unlikely to be a useful hash function in practice.

Recent results on hashing algorithms for maximum inner product search~\cite{Proc:Shrivastava_NIPS14} have shown the usefulness of asymmetric transformations in constructing provable hash functions for new similarity measures, which were otherwise impossible.  Going further along this line, we provide a novel (and still very simple) asymmetric transformation for binary data, that corrects minhash and removes the undesirable bias of minhash towards the sizes of the sets involved. Such an asymmetric correction eventually leads to a provable hashing scheme for binary inner products, which we call {\em asymmetric minwise hashing (MH-ALSH}). Our theoretical comparisons show that for binary data, which are common over the web, the new hashing scheme is provably more efficient that the recently proposed asymmetric hash functions for general inner products~\cite{Proc:Shrivastava_NIPS14,Article:Shrivastava_arXiv14}. Thus, we obtain a provable algorithmic improvement over the state-of-the-art hashing technique for binary inner products. The construction of our asymmetric transformation for minhash could be of independent interest in itself.

The proposed asymmetric minhash significantly outperforms existing hashing schemes, in the tasks of ranking and near neighbor search with Jaccard containment as the similarity measure, on four real-world  high-dimensional datasets. Our final proposed algorithm is simple and only requires very small modifications of the traditional minhash and hence it can be easily adopted in practice. 

%
%

\section{Background}

\subsection{$c$-Approximate Near Neighbor Search and Classical LSH}

Past attempts of finding efficient algorithms, for exact near
neighbor search based on space partitioning, often turned out to be a
disappointment with the massive dimensionality of modern datasets~\cite{Proc:Weber_VLDB98}. Due to the curse of dimensionality, theoretically it is hopeless to
obtain an efficient algorithm for exact near neighbor search. Approximate
versions of near neighbor search problem were proposed~\cite{Proc:Indyk_STOC98}
to overcome the linear query time bottleneck. One  commonly
adopted such formulation is the $c$-approximate Near Neighbor ($c$-NN).

\begin{mydef}
($c$-Approximate Near Neighbor or $c$-NN).~\cite{Proc:Indyk_STOC98} Given a
set of points in a $d$-dimensional space $\mathbb{R}^d$, and
parameters $S_0 > 0$, $\delta > 0$, construct a data structure which,
given any query point q, does the following with probability $1-
\delta$: if there exist an $S_0$-near neighbor of q in P, it reports
some $cS_0$-near neighbor.
\end{mydef}

The usual notion of $S_0$-near neighbor is in terms of distance. Since we are dealing with similarities, we define $S_0$-near
neighbor of point $q$ as a point $p$ with $Sim(q,p) \ge S_0$, where
$Sim$ is the similarity function of interest.

The  popular technique, with near optimal guarantees for $c$-NN in
many interesting cases, uses the underlying theory of \emph{Locality
Sensitive Hashing} (LSH)~\cite{Proc:Indyk_STOC98}. LSH are family of functions,
with the property that similar input objects in the domain of these
functions have a higher probability of colliding in the range space
than non-similar ones. More specifically, consider $\mathcal{H}$ a
family of hash functions mapping $\mathbb{R}^D$ to some set
$\mathcal{S}$.

\begin{mydef} \label{def:LSH}(Locality Sensitive Hashing)\ A family $\mathcal{H}$ is
called $(S_0,cS_0,p_1,p_2)$ sensitive if for any two point $x,y \in
\mathbb{R}^D$ and $h$ chosen uniformly from $\mathcal{H}$ satisfies
the following:
\begin{itemize}
\item if $Sim(x,y)\ge S_0$ then $Pr_\mathcal{H}(h(x) = h(y)) \ge p_1$
\item if $ Sim(x,y)\le cS_0$ then $Pr_\mathcal{H}(h(x) = h(y)) \le p_2$
\end{itemize}
\end{mydef}

For approximate nearest neighbor search typically, $p_1 > p_2$ and $c
< 1$ is needed. Note, $c<1$ as we are defining neighbors in terms of
similarity. To obtain distance analogy we can resort to $D(x,y) = 1-
Sim(x,y)$
\begin{fact}
\label{fct}
~\cite{Proc:Indyk_STOC98} Given a family of $(S_0,cS_0,p_1,p_2)$ -sensitive
hash functions, one can construct a data structure for $c$-NN with
$O(n^\rho \log_{1/p_2}{n})$ query time and space $O(n^{1+\rho})$,
$\rho = \frac{\log{1/p_1}}{\log{1/p_2}} < 1$
\end{fact}
LSH trades off query time with extra preprocessing time and space that
can be accomplished off-line. It requires constructing a one time data
structure which costs $O(n^{1 + \rho})$ space and further any
$c$-approximate near neighbor queries can be answered in $O(n^\rho
\log_{1/p_2}{n})$ time in the worst case.\\

A particularly interesting sufficient condition for existence of LSH
is the monotonicity of the collision probability in $Sim(x,y)$. Thus,
if a hash function family $\mathcal{H}$ satisfies,
\begin{align}
Pr_{h\in\mathcal{H}}(h(x) = h(y)) = g(Sim(x,y)),
\end{align}
where $g$ is any strictly monotonically increasing function, then the
conditions of Definition~\ref{def:LSH} are automatically satisfied for all $c
< 1$.

The quantity $\rho < 1$ is a property of the LSH family, and it is of particular interest because it determines the worst case query complexity of the $c$-approximate near neighbor search. It should be further noted, that the complexity depends on $S_0$ which is the operating threshold and $c$, the approximation ratio
we are ready to tolerate. In case when we have two or more LSH families for a given similarity measure, then the LSH family with smaller value of $\rho$, for given $S_0$ and $c$, is preferred.

\subsection{Minwise Hashing (Minhash)}
\label{sec:minhash}
Minwise hashing~\cite{Proc:Broder} is the LSH for the {\em resemblance}, also known as the {\em Jaccard similarity}, between sets.  In this paper, we focus on binary data vectors which can be equivalent viewed as sets.

Given a set $x\in\Omega = \{1, 2, ..., D\}$, the minwise hashing family applies a random permutation $\pi:\Omega
\rightarrow \Omega$ on  $x$ and stores only the minimum
value after the permutation mapping. Formally minwise hashing (or minhash) is defined as:
\begin{equation}\label{eq:min}
h_{\pi}(x) = \min(\pi(x)).
\end{equation}

Given sets $x$ and $y$, it can be shown that  the probability  of collision is the resemblance $\mathcal{R} =\frac{|x \cap y|}{| x
\cup y|}$:
\begin{equation}
\label{eq:minhash}
Pr_{\pi}({h_{\pi}(x) = h_{\pi}(y)) = \frac{|x \cap y|}{| x
\cup y|}} = \frac{a}{f_x +f_y -a} = \mathcal{R}.
\end{equation}
where $f_x = |x|$, $f_y = |y|$, and $a = |x\cap y|$.  It follows from Eq. (~\ref{eq:minhash}) that minwise hashing is
$(S_0,cS_0,S_0,cS_0)$-sensitive family of hash function when the
similarity function of interest is resemblance.\\

Even though minhash was really meant for retrieval with resemblance similarity, it is nevertheless  a popular hashing scheme used for retrieving set containment or intersection for binary data~\cite{Proc:Agrawal_SIGMOD2010}. In practice, the ordering of inner product $a$ and the ordering or resemblance $\mathcal{R}$ can be  different because of the variation in the values of $f_x$ and $f_y$, and as argued in Section~\ref{sec:intro}, which may be undesirable and lead to suboptimal results.  We show later that by exploiting asymmetric transformations we can get away with the undesirable dependency on the number of nonzeros leading to a better hashing scheme for indexing set intersection (or binary inner products).

\subsection{LSH for L2 Distance (L2LSH)}
\label{sec:L2Hash}
\cite{Proc:Datar_SCG04} presented a novel LSH family for all $L_p$ ($p
\in (0,2]$) distances. In particular, when $p =2$, this scheme
provides an LSH family for $L_2$ distance. Formally, given a fixed
number $r$, we choose a random vector $w$ with each component
generated from i.i.d. normal, i.e., $w_i \sim N(0,1)$, and a scalar
$b$ generated uniformly at random from $[0,r]$. The hash function is
defined as:

\begin{equation}\label{eq:L2Hash} h_{w,b}^{L2}(x) = \left\lfloor
\frac{w^Tx + b}{r} \right\rfloor, \end{equation}
where $\lfloor \rfloor$ is the floor operation. The collision
probability under this scheme can be shown to be

\begin{equation}\label{eq:L2collprob}Pr(h^{L2}_{w,b}(x) =
h^{L2}_{w,b}(y)) = F_r(d),\end{equation}
\begin{equation}\label{eq:F}
F_r(d) = 1 - 2\Phi(-r/d) - \frac{2}{\sqrt{2\pi}r/d}\left(1 -
e^{-r^2/(2d^2)}\right)
\end{equation}
where $\Phi(x) =
\int_{-\infty}^{x}\frac{1}{\sqrt{2\pi}}e^{-\frac{x^2}{2}}dx$ is the
cumulative density function (cdf) of standard normal distribution and
$d = ||x -y||_2$ is the Euclidean distance between the vectors $x$ and
$y$. This collision probability $F(d)$ is a monotonically decreasing
function of the distance $d$ and hence $h^{L2}_{w,b}$ is an LSH for
$L2$ distances. This scheme is also the part of LSH
package~\cite{Report:E2LSH}. Here $r$ is a parameter.

\subsection{LSH for Cosine Similarity (SRP)}

Signed Random Projections (SRP) or {\em simhash} is another popular LSH for the cosine similarity measure, which originates from the concept of \emph{\bf Signed Random Projections (SRP)}~\cite{Article:Goemans_JACM95,Proc:Charikar}. Given a vector $x$, SRP utilizes a random $w$ vector with each component generated from i.i.d. normal, i.e., $w_i \sim N(0,1)$, and only stores the sign of the projection. Formally simhash is given by
 \begin{align}\label{eq:signhash}
h^{sign}(x) = sign(w^Tx).
\end{align}
It was shown in the seminal work~\cite{Article:Goemans_JACM95} that collision under SRP satisfies the following equation:
\begin{equation}
\label{eq:srp}
Pr_w(h^{sign}(x) = h^{sign}(y)) = 1 - \frac{\theta}{\pi},
\end{equation}
where  $\theta = cos^{-1}\left( \frac{x^Ty}{||x||_2 ||y||_2}\right)$. The term $\frac{x^Ty}{||x||_2 ||y||_2}$ is the popular {\bf cosine similarity}.

For sets (or equivalently binary vectors), the { cosine similarity} reduces to
\begin{equation}
\mathcal{S} =  \frac{a}{\sqrt{f_xf_y}}.
\end{equation}

\vspace{0.05in}

The recent work on {\em coding for random projections}~\cite{Proc:Li_ICML14,Report:RPCodeLSH2014} has shown the advantage of SRP (and 2-bit random projections) over L2LSH for both similarity estimation and near neighbor search. Interestingly, another recent work~\cite{Proc:Shrivastava_AISTATS14} has shown that for binary data (actually even sparse non-binary data), minhash can significantly outperform SRP for near neighbor search even as we evaluate both SRP and minhash in terms of the cosine similarity (although minhash is designed for resemblance). This motivates us to design asymmetric minhash for achieving better performance in retrieving set containments. But first, we provide an overview of asymmetric LSH for general inner products (not restricted to binary data).

\subsection{Asymmetric LSH (ALSH) for General Inner Products} \label{sec:ALSH}

The term ``ALSH'' stands for {\em asymmetric LSH}, as used in a recent work~\cite{Proc:Shrivastava_NIPS14}. Through  an elementary argument,  \cite{Proc:Shrivastava_NIPS14} showed that it is not possible to have a Locality Sensitive Hashing (LSH) family for general unnormalized inner products.

For inner products between vectors $x$ and $y$, it is possible to have $x^Ty >> x^Tx$. Thus for any hashing scheme $h$ to be a valid LSH, we must have $Pr(h(x) = h(y)) > Pr(h(x) = h(x)) = 1$, which is an impossibility. It turns out that there is a simple fix, if we allow asymmetry in the hashing scheme. Allowing asymmetry leads to an extended framework of asymmetric locality sensitive hashing (ALSH). The idea to is have a different hashing scheme for assigning buckets to the data point in the collection $\mathcal{C}$, and an altogether different hashing scheme while querying.\\

\noindent\textbf{Definition:} ({\bf\it Asymmetric} Locality Sensitive
Hashing (ALSH))\ A family $\mathcal{H}$, along with the two vector
functions $Q:\mathbb{R}^D \mapsto \mathbb{R}^{D'}$ ({\bf Query
Transformation}) and $P:\mathbb{R}^D \mapsto \mathbb{R}^{D'}$ ({\bf
Preprocessing Transformation}), is called
$(S_0,cS_0,p_1,p_2)$-sensitive if for a given $c$-NN instance with
query $q$, and the hash function $h$ chosen uniformly from
$\mathcal{H}$ satisfies the following:
\begin{itemize}
\item if $Sim(q,x)\ge S_0$ then $Pr_\mathcal{H}(h(Q(q))) = h(P(x))) \ge p_1$
\item if $ Sim(q,x)\le cS_0$ then $Pr_\mathcal{H}(h(Q(q)) = h(P(x))) \le p_2$
\end{itemize}
Here $x$ is any point in the collection $\mathcal{C}$. Asymmetric LSH borrows all theoretical guarantees of the LSH.

\begin{fact}\label{theo:extendedLSH}
Given a family of hash function $\mathcal{H}$ and the associated query
and preprocessing transformations $Q$ and $P$ respectively, which is
$(S_0,cS_0,p_1,p_2)$ -sensitive, one can construct a data structure
for $c$-NN with $O(n^\rho \log{n})$ query time and space $O(n^{1 +
\rho})$, where $\rho = \frac{\log{p_1}}{\log{p_2}}$.
\end{fact}

~\cite{Proc:Shrivastava_NIPS14} showed that using asymmetric transformations, the problem of {\bf maximum inner product search (MIPS)} can be reduced to the problem of approximate near neighbor search in $L_2$.  The algorithm first starts by scaling all $x \in \mathcal{C}$ by a constant large enough, such that $||x||_2 \le U < 1$.  The proposed  ALSH family ({\bf L2-ALSH}) is the  LSH family for $L_2$ distance with the Preprocessing transformation $P:\mathbb{R}^D \mapsto \mathbb{R}^{D+m}$ and the Query transformation
$Q:\mathbb{R}^D \mapsto \mathbb{R}^{D+2m}$ defined as follows:
\begin{align}
\label{eq:L2P}P^{L2}(x) &= [x; ||x||^2_2; ....;||x||^{2^m}_2;1/2;...;1/2]\\
\label{eq:L2Q}Q^{L2}(x) &= [x; 1/2;...;1/2;||x||^2_2; ....;||x||^{2^m}_2],
\end{align}
where [;] is the concatenation. $P^{L2}(x)$ appends $m$ scalers of the form
$||x||_2^{2^i}$ followed by $m$ ``1/2s" at the end of the vector $x$, while $Q^{L2}(x)$ first
appends $m$ ``1/2s'' to the end of the vector $x$ and then $m$ scalers of the form
$||x||_2^{2^i}$.  It was shown that this leads to provably efficient algorithm for MIPS.

\begin{fact}
\label{fact:L2-ALSH}~\cite{Proc:Shrivastava_NIPS14}  For the problem of $c$-approximate MIPS in a bounded space, one can construct a data structure having \\ $O(n^{\rho_{L2-ALSH}^*} \log{n})$ query time and space $O(n^{1+\rho_{L2-ALSH}^*})$, where $\rho_{L2-ALSH}^* < 1$  is the solution to constrained optimization (\ref{eq:optrho}).
\begin{align}
\label{eq:optrho}
\rho_{L2-ALSH}^* &= \min_{U <1,m \in N, r} \frac{\log{F_r\big(\sqrt{m/2- 2S_0\left(\frac{U^2}{V^2}\right) + 2U^{2^{m+1}}}\big)}}{\log{ F_r\big(\sqrt{m/2 - 2cS_0\left(\frac{U^2}{V^2}\right)}\big)}} \hspace{0.1in} \\ \notag
 s.t. &\hspace{0.1in} \frac{U^{(2^{m+1}-2)}V^2}{S_0} <  1 -  c, \hspace{0.05in}
\end{align}
\end{fact}
Here the guarantees depends on the maximum norm of the space $V = \max_{x \in \mathcal{C}} ||x||_2$.\\

Quickly, it was realized that a very similar idea can convert the MIPS problem in the problem of maximum cosine similarity search which can be efficiently solve by SRP leading to a new and better ALSH for MIPS {\bf Sign-ALSH}~\cite{Article:Shrivastava_arXiv14} which works as follows:
The algorithm again first starts by scaling all $x \in \mathcal{C}$ by a constant large enough, such that $||x||_2 \le U < 1$.  The proposed  ALSH family ({\bf Sign-ALSH}) is the SRP  family for cosine similarity  with the Preprocessing transformation $P^{sign}:\mathbb{R}^D \mapsto \mathbb{R}^{D+m}$ and the Query transformation
$Q^{sign}:\mathbb{R}^D \mapsto \mathbb{R}^{D+2m}$ defined as follows:
\begin{align}
\label{eq:signP} P^{sign}(x) &= [x; 1/2  - ||x||^2_2; ...;1/2 - ||x||^{2^m}_2;0;...;0]\\
\label{eq:signQ}Q^{sign}(x) &= [x; 0; ...; 0; 1/2  - ||x||^2_2; ...;1/2 - ||x||^{2^m}_2],
\end{align}
where [;] is the concatenation. $P^{sign}(x)$ appends $m$ scalers of the form
$1/2 - ||x||_2^{2^i}$ followed by $m$ ``0s" at the end of the vector $x$, while $Q^{sign}(x)$
appends $m$ ``0'' followed by $m$ scalers of the form
$1/2 - ||x||_2^{2^i}$ to the end of the vector $x$.  It was shown that this leads to provably efficient algorithm for MIPS. \\

As demonstrated by the recent work~\cite{Proc:Li_ICML14} on {\em coding for random projections}, there is a significant advantage of SRP over L2LSH for near neighbor search. Thus, it is not surprising that Sign-ALSH outperforms L2-ALSH for the MIPS problem.\\

Similar to L2LSH, the runtime guarantees for Sign-ALSH can be shown as:
\begin{fact}
\label{theo:main} For the problem of
$c$-approximate MIPS, one can construct a data structure having
$O(n^{\rho_{Sign-ALSH}^*} \log{n})$ query time and space $O(n^{1+\rho_{Sign-ALSH}^*})$, where
$\rho_{Sign-ALSH}^* < 1$ is the solution to constraint optimization problem
\begin{align}
\label{eq:optrho_u}
 &\rho_{Sign-ALSH}^* = \min_{U,m,} \frac{\log{\bigg(1-\frac{1}{\pi}\cos^{-1}\bigg(\frac{ S_0 \times\left(\frac{U^2}{V^2}\right)}{\frac{m}{4} + U^{2^{m+1}}}\bigg)\bigg)}}{\log{\bigg(1-\frac{1}{\pi}\cos^{-1} \bigg(\min\{\frac{cS_0U^2}{V^2},z^*\}\bigg)\bigg)}}\\
\nonumber  &z^* = \left[\frac{(m - m2^{m-1}) + \sqrt{(m - m2^{m-1})^2 + m^2(2^m-1)}}{4(2^m-1)}\right]^{2^{-m}}
\end{align}
\end{fact}

There is a similar asymmetric transformation~\cite{Proc:Bachrach_recsys14,Report:Neyshabur_arXiv14} which followed by signed random projection leads to another ALSH having very similar performance to Sign-ALSH. The $\rho$ values, which were also very similar to the $\rho_{Sign-ALSH}$ can be shown as
\begin{align}
\label{eq:sqrt-rho}
\rho_{Sign} &=  \frac{\log{\bigg(1-\frac{1}{\pi}\cos^{-1}\bigg(\frac{S_0}{V^2}\bigg)\bigg)}}{\log{\bigg(1-\frac{1}{\pi}\cos^{-1} \bigg(\frac{cS_0}{V^2}\bigg)\bigg)}}
\end{align}

Both L2-ALSH and Sign-ALSH work for any general inner products over $\mathbb{R}^D$. For sparse and high-dimensional binary dataset which are common over the web, it is known that minhash is typically the preferred choice of hashing over random projection based hash functions~\cite{Proc:Shrivastava_AISTATS14}. We show later that the ALSH derived from minhash, which we call asymmetric minwise hashing ({\em MH-ALSH}), is more suitable for indexing set intersection for sparse binary vectors than the existing ALSHs for general inner products.

\section{A Construction of LSH for Indexing Binary Inner Products}
\label{sec:LSHIP}
In~\cite{Proc:Shrivastava_NIPS14}, it was shown that there cannot exist any LSH for general
unnormalized inner product. The key argument used in the proof was the
fact that it is possible to have $x$ and $y$ with $x^Ty \gg x^Tx$.  However, binary inner product (or set intersection) is  special. For any
two binary vectors $x$ and $y$ we always have $x^Ty \le x^Tx$.
Therefore, the argument used to show non-existence of LSH for general inner products
does not hold true any more for this special case. In fact, there does exist an LSH for
binary inner products (although it is mainly for theoretical interest). We provide an explicit construction in this
section.\\

Our proposed LSH construction is based on sampling. Simply sampling a
random component leads to the popular LSH for hamming distance~\cite{Book:Raj_Ullman}. The ordering of inner product is different from
that of hamming distance. The hamming distance between $x$ and query $q$ is given by $f_x + f_q
- 2a$, while we want the collision probability to be monotonic in the
inner product $a$. $f_x$ makes it non-monotonic in $a$. Note that $f_q$ has no effect on ordering of $x \in \mathcal{C}$ because it is constant for every query. To construct an LSH monotonic in binary inner product, we need an extra trick.

Given a binary data vector $x$, we sample a random co-ordinate (or
attribute). If the value of this co-ordinate is $1$ (in other words if
this attribute is present in the set), our hash value is a fixed
number $0$. If this randomly sampled co-ordinate has value $0$ (or the
attribute is absent) then we independently generate a random integer uniformly from $\{1,
2, ..., N\}$. Formally,
\begin{equation}
\mathcal{H}_S(x) = \begin{cases} 0 &\text{ if $x_i = 1$, $i$ drawn uniformly}\\
\text{rand(1,N)} &\text{otherwise}
\end{cases}
\end{equation}
\begin{theorem}
Given two binary vectors $x$ and $y$, we have
\begin{equation}
Pr( \mathcal{H}_S(x)= \mathcal{H}_S(y)) =
\left[\frac{N-1}{N}\right]\frac{a}{D} + \frac{1}{N}
\end{equation}
\end{theorem}
\begin{proof}
The probability that both ${H}_S(x)$ and ${H}_S(y)$ have value 0 is $\frac{a}{D}$. The only other way both can be equal is when the two independently generated random numbers become equal, which happens with probability $\frac{1}{N}$. The total probability is $(\frac{a}{D} + (1 -\frac{a}{D})\frac{1}{N})$ which simplifies to the desired expression.
\end{proof}
\begin{corollary}
$\mathcal{H}_S$ is
$(S_0,cS_0,\left[\frac{N-1}{N}\right]\frac{S_0}{D} +
\frac{1}{N}, \left[\frac{N-1}{N}\right]\frac{cS_0}{D} +
\frac{1}{N})$-sensitive locality sensitive hashing for binary
inner product with $\rho_{\mathcal{H}_S} =
\frac{\log{\left(\left[\frac{N-1}{N}\right]\frac{S_0}{D} +
\frac{1}{N}\right)}}{\log{\left(\left[\frac{N-1}{N}\right]\frac{cS_0}{D}
+ \frac{1}{N}\right)}} < 1$
\end{corollary}

\subsection{Shortcomings}

The above LSH for binary inner product is likely to be very
inefficient for sparse and high dimensional datasets. For those
 datasets, typically the value of $D$ is very high and the sparsity
ensures that $a$ is very small. For modern web datasets, we can have
$D$ running into billions (or $2^{64}$) while the sparsity is only in few
hundreds or perhaps thousands~\cite{Report:chandra_10}. Therefore, we have $\frac{a}{D} \simeq 0$ which
essentially boils down to $\rho_{\mathcal{H}_S} \simeq 1$. In other words, the
hashing scheme becomes worthless in sparse high dimensional domain. On
the other hand, if we observe the collision probability of minhash
Eq.(~\ref{eq:minhash}), the denominator is $f_x +f_y - a$, which is usually of the
order of $a$ and much less than the dimensionality for sparse
datasets.

Another way of realizing the problem with the above LSH is to note that it is informative only if a randomly sampled co-ordinate has value equal to 1. For very sparse dataset with $a \ll D$, sampling a non zero coordinate has probability $\frac{a}{D} \simeq 0$. Thus, almost all of the hashes will be independent random numbers.

\subsection{Why Minhash Can Be a Reasonable Approach?}

In this section, we argue why retrieving inner product based on plain  minhash is a reasonable thing to do. Later, we will show a provable way to improve it using asymmetric transformations.

 The number of nonzeros in the query, i.e., $|q| = f_q$ does not change the identity of  $\arg\max$ in Eq.(\ref{eq:prob}).  Let us assume that we have data of bounded sparsity and define constant $M$ as
\begin{align}
M = \max_{x \in \mathcal{C}} |x|
\end{align}
where $M$ is simply the maximum number of nonzeros (or maximum cardinality of sets) seen in the database.  For  sparse data seen in practice $M$ is likely to be  small compared to $D$.
Outliers, if any, can be handled separately.  By observing that  $a \le f_x \le M$,  we also have
\begin{align}
 \frac{a}{f_q + M -a}\le \frac{a}{f_x + f_q -a} = \mathcal{R} \le \frac{a}{f_q}
\end{align}
Thus, given the bounded sparsity, if we assume that the number of nonzeros in the query is given, then we can show that minhash is an LSH for inner products $a$ because the collision probability can be upper and lower bounded by purely functions of $a, M$ and $f_q$.

\begin{theorem}
\label{theo:minhash}
Given bounded sparsity and query $q$ with $|q| = f_q$, minhash is a $(S_0,cS_0, \frac{S_0}{f_q + M -S_0},\frac{cS_0}{f_q})$ sensitive for inner products $a$ with $\rho_{min}^q = \frac{\log{\frac{S_0}{f_q + M -S_0}}}{\log{\frac{cS_0}{f_q}}}$
\end{theorem}
This explains why minhash might be a reasonable hashing approach for retrieving inner products or set intersection.

Here, if we remove the assumption that  $|q| = f_q$ then in the worst case  $\mathcal{R} \le \frac{a}{f_q} \le 1$ and we get $\log{1}$ in the denominator. Note that the above is the worst case analysis and the assumption  $|q| = f_q$ is needed to obtain any meaningful $\rho$ with minhash. We show the power of ALSH in the next section, by providing a  better hashing scheme and we do not even need the assumption of fixing $|q| = f_q$.

\section{Asymmetric Minwise Hashing (MH-ALSH)}

In this section, we provide a very simple asymmetric fix to minhash, named {\em asymmetric minwise hashing (MH-ALSH)}, which makes the
overall collision probability monotonic in the original inner product $a$. For sparse binary data, which is common in
practice, we later show that the proposed hashing scheme is superior (both
theoretically as well as empirically) compared to the existing ALSH
schemes for inner product~\cite{Proc:Shrivastava_NIPS14}.

\subsection{The New ALSH for Binary Data}
\label{sec:Amin}

We define the new preprocessing and query transformations $P': [0,1]^D
\rightarrow [0,1]^{D+M}$ and $Q': [0,1]^D \rightarrow [0,1]^{D+M}$ as:
\begin{align}
\label{eq:PAmin} P'(x)&= [x;1;1;1;...;1;0;0;...;0]\\
\label{eq:QAmin}  Q'(x)&= [x;0;0;0;...;0],
\end{align}
where [;] is the concatenation to vector $x$.  For $P'(x)$ we append $M - f_x$
1s and rest  $f_x$ zeros, while in $Q'(x)$  we simply append $M$ zeros.

At this point we can already see the power of asymmetric transformations. The original
inner product between $P'(x)$ and $Q'(x)$ is unchanged and its value is $a = x^Ty$. Given the query
$q$, the new resemblance $R'$ between $P'(x)$ and $Q'(q)$ is
\begin{align}
\label{eq:R'}
R' = \frac{|P'(x) \cap Q'(q)|}{|P'(x) \cup Q'(q)|} = \frac{a}{M + f_q - a}.
\end{align}
If we define our new similarity as $Sim(x,y) = \frac{a}{M + f_q - a}$, which is similar in nature to the containment $\frac{a}{f_q}$, then the near neighbors in this new similarity are the same as near neighbors with respect to either set intersection $a$ or set containment
$\frac{a}{f_q}$. Thus, we can instead compute near neighbors in $\frac{a}{M + f_q - a}$ which is also the resemblance between $P'(x)$ and $Q'(q)$. We can therefore use minhash on $P'(x)$ and $Q'(q)$.

Observe that now we have $M + f_q - a$ in the denominator, where $M$
is the maximum nonzeros seen in the dataset (the cardinality of
largest set), which for very sparse data is likely to be much smaller
than $D$. Thus, asymmetric minhash is a better scheme than
$\mathcal{H}_S$ with collision probability roughly $\frac{a}{D}$ for
very sparse datasets where we usually have $M \ll D$. This is an interesting example where we do have an LSH scheme but an altogether different asymmetric LSH (ALSH) improves over existing LSH. This is not surprising because asymmetric LSH families are more powerful~\cite{Proc:Shrivastava_NIPS14}.

From theoretical perspective, to obtain an upper bound on the query and
space complexity of $c$-approximate near neighbor with binary inner products, we want the collision probability to be independent
of the quantity $f_q$. This is not difficult to achieve. The
asymmetric transformation used to get rid of $f_x$ in the denominator
can be reapplied to get rid of $f_q$. \\

Formally, we can define $P'':
[0,1]^D \rightarrow [0,1]^{D+2M}$ and $Q'': [0,1]^D \rightarrow
[0,1]^{D+2M}$ as :
\begin{align}
\label{eq:P''}
P''(x) = Q'(P'(x)) ; \ \ \ \ \ Q''(x) = P'(Q'(x)) ;
\end{align}
where in $P''(x) $  we append $M - f_x$ 1s and rest $M
+ |f_x|$ zeros, while in $Q''(x)$ we append $M$ zeros, then $M
-f_q$ 1s and rest zeros

Again the inner product $a$ is unaltered, and the new resemblance then becomes
\begin{align}
\label{eq:collAmin}
R'' = \frac{|P''(x) \cap Q''(q)|}{|P''(x) \cup Q''(q)|} = \frac{a}{2M - a}.
\end{align}
which is independent of $f_q$ and is monotonic in $a$. This allows us to
achieve a formal upper bound on the complexity of $c$-approximate maximum
inner product search with the new asymmetric minhash.

From the collision probability expression, i.e., Eq. (\ref{eq:collAmin}), we have
\begin{theorem}
Minwise hashing along with Query transformation $Q''$ and Preprocessing transformation $P''$ defined by Equation~\ref{eq:P''} is a $(S_0,cS_0, \frac{S_0}{2M - S_0}, \frac{cS_0}{2M - cS_0})$ sensitive asymmetric hashing family for set intersection.
\end{theorem}

This leads to an important corollary.
\begin{corollary}
There exist an algorithm for $c$-approximate set intersection (or binary inner products), with bounded sparsity $M$, that requires $O(n^{1+ \rho_{MH-ALSH}})$ space and $O(n^\rho_{MH-ALSH} \log{n})$, where
\begin{align}
\rho_{MH-ALSH} = \frac{\log{\frac{S_0}{2 M -
S_0}}}{\log{\frac{cS_0}{2M - cS_0}}} < 1
\end{align}
\end{corollary}

Given query $q$ and any point $x \in \mathcal{C}$, the collision probability under traditional minhash is $R = \frac{a}{f_x + f_q - a}$. This penalizes sets with high $f_x$, which in many scenarios  not desirable. To balance this negative effect, asymmetric transformation penalizes sets with smaller $f_x$. Note, that $M- f_x$ ones added in the transformations $P'(x)$ gives  additional chance in proportion to $M-f_x$ for minhash of $P'(x)$ not to match with the minhash of $Q'(x)$. This asymmetric probabilistic correction balances the penalization  inherent in minhash. This is a simple way of correcting the probability of collision which could be of independent interest in itself. We will show in our evaluation section, that despite this simplicity such correction leads to significant improvement over plain minhash.

\subsection{Faster Sampling}

Our transformations $P''$ and $Q''$ always create sets with $2M$ nonzeros. In case when $M$ is big, hashing might take a lot of time. We can use fast consistent weighted sampling~\cite{Report:Manasse_00,Proc:Ioffe_ICDM10} for efficient generation of hashes. We can instead use transformations  $P'''$ and $Q'''$ that makes the data non-binary as follows
\begin{align}
P'''(x) &= [x ; M- f_{x}; 0]\\\notag
Q'''(x) &=[x; 0 ; M - {f_x}]\notag
\end{align}
It is not difficult to see that the weighted Jaccard similarity (or weighted resemblance) between $P'''(x)$ and $Q'''(q)$ for given query $q$ and any $x \in \mathcal{C}$ is
\begin{align}
\mathcal{R}_W = \frac{\sum_i \min(P'''(x)_i, Q'''(q)_i)  }{\sum_i \max(P'''(x)_i, Q'''(q)_i) } =  \frac{a}{2M - a}.
\end{align}
Therefore, we can use fast consistent weighted sampling for weighted Jaccard similarity on $P'''(x)$ and $Q'''(x)$ to compute the hash values in time constant per nonzero weights, rather than maximum sparsity $M$. In practice we will need many hashes for which we can  utilize the recent line of work that make minhash and weighted minhash significantly much faster~\cite{Proc:OneHashLSH_ICML14,Report:Haeupler_arXiv14}.


\newpage

\section{Theoretical Comparisons}

For solving the MIPS problem in general  data types, we already know two asymmetric hashing schemes, {\em L2-ALSH} and {\em Sign-ALSH}, as described in Section~\ref{sec:ALSH}. In this section, we provide
theoretical comparisons of the two existing ALSH methods with the proposed asymmetric minwise hashing ({\em MH-ALSH}). As argued, the LSH scheme described in
Section~\ref{sec:LSHIP} is  unlikely to be useful in practice because of its dependence on $D$; and hence we safely ignore it for simplicity of the discussion.

Before we formally compare various asymmetric LSH schemes for  maximum inner product search, we argue why asymmetric minhash  should be advantageous over traditional  minhash for retrieving inner products. Let $q$ be the binary query vector, and $f_q$ denotes the number of
nonzeros in the query. The $\rho_{MH-ALSH}$ for asymmetric minhash in
terms of $f_q$ and $M$ is straightforward from the collision
probability Eq.(\ref{eq:R'}):
\begin{align}
\rho_{MH-ALSH}^q = \frac{\log{\frac{S_0}{f_q + M -
S_0}}}{\log{\frac{cS_0}{f_q + M - cS_0}}}
\end{align}
 For minhash, we have from theorem~\ref{theo:minhash} $\rho_{min}^q = \frac{\log{\frac{S_0}{f_q + M -S_0}}}{\log{\frac{cS_0}{f_q}}}$. Since $M$ is the upper bound on the sparsity and $cS_0$ is some value of inner product, we have $M - cS_0 \ge 0$. Using this fact, the following theorem immediately follows
\begin{theorem}
For any query q, we have  $\rho_{MH-ALSH}^q \le \rho_{min}^q$.
\end{theorem}
This result theoretically explains why asymmetric minhash is better for retrieval with binary inner products, compared to plain minhash.

\begin{figure}[!h]
\begin{center}
\mbox{
\hspace{-0.1in}\includegraphics[width=3.5in]{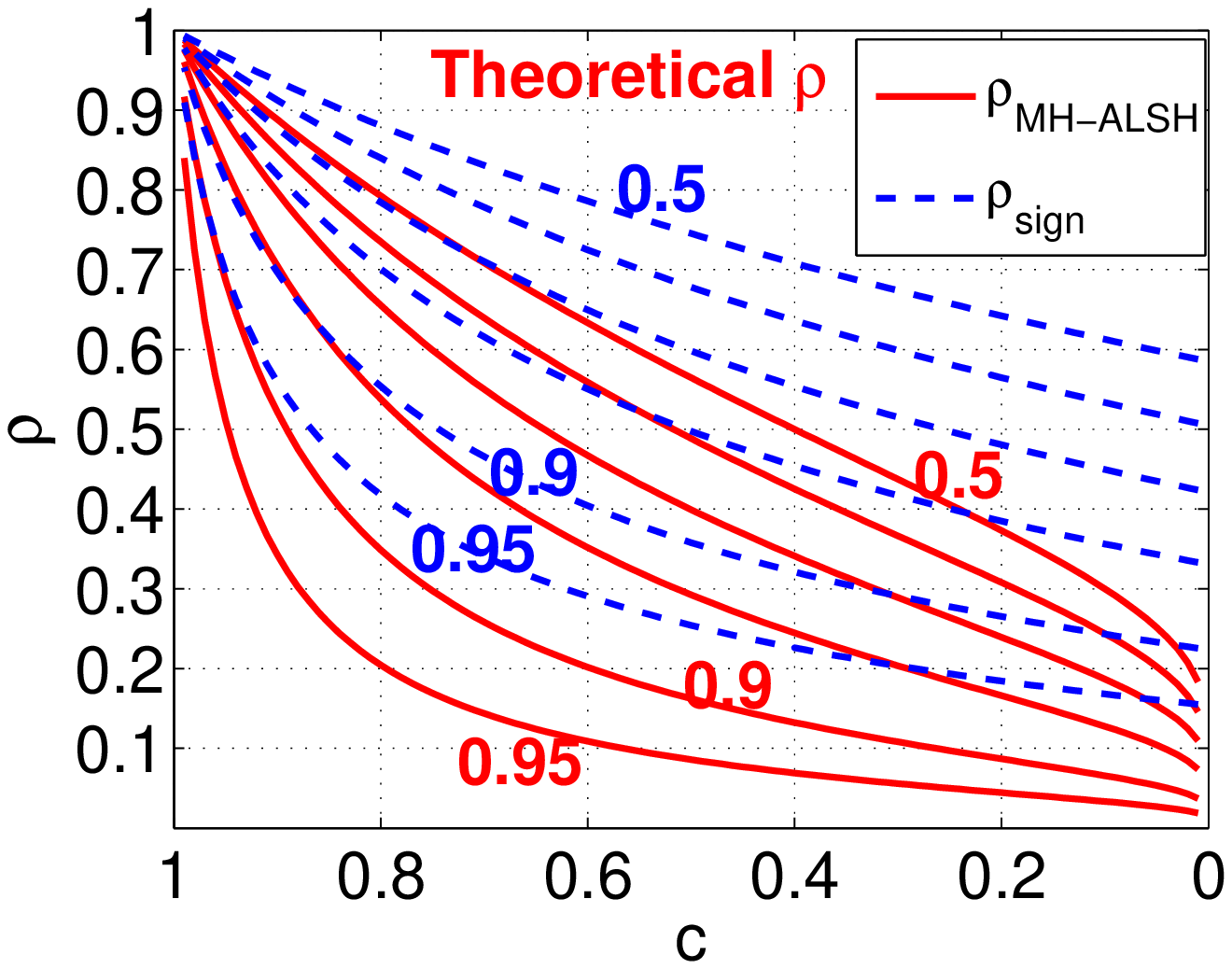}
\hspace{-0.1in}\includegraphics[width=3.5in]{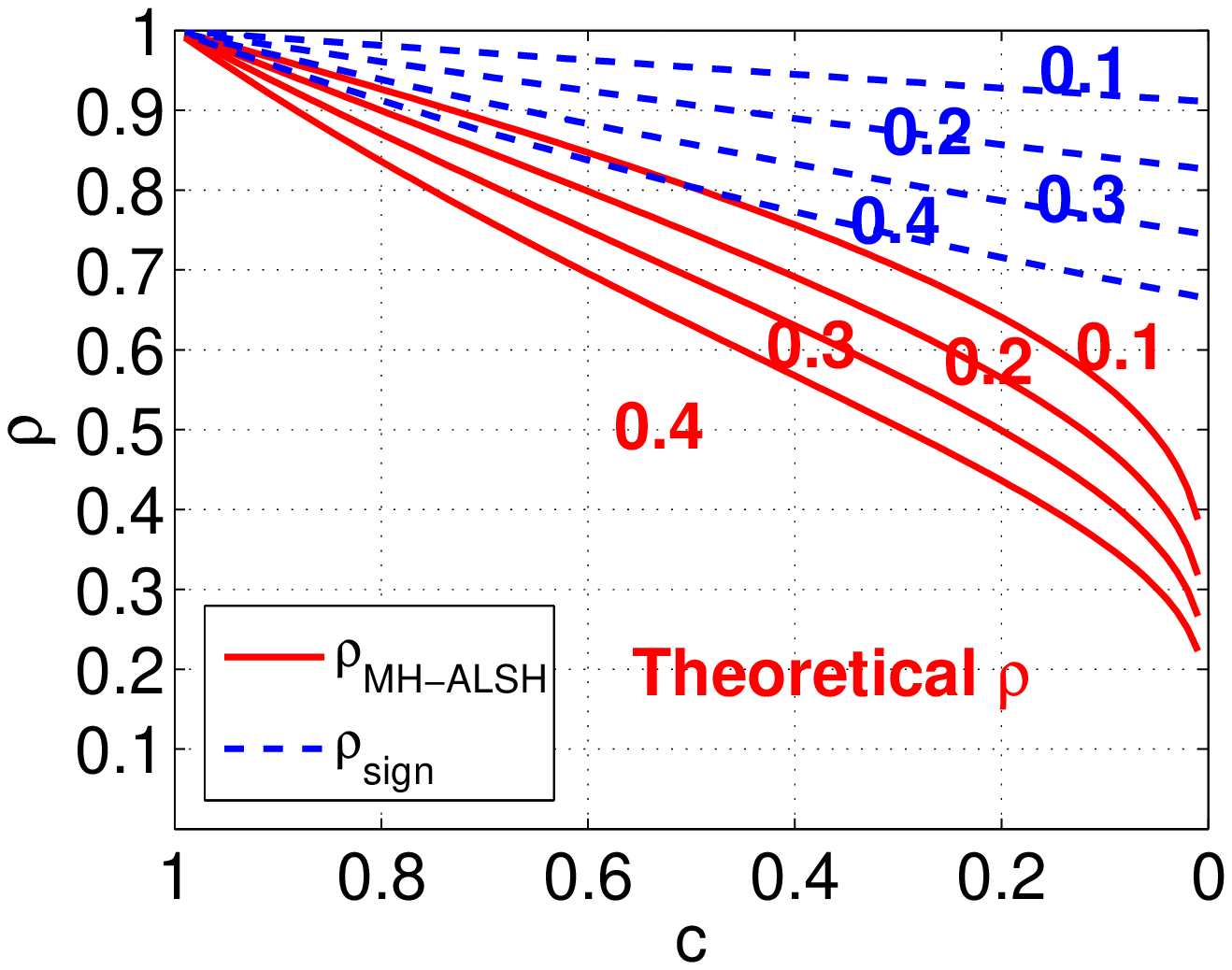}
}
\end{center}
\vspace{-0.2in}
\caption{Values of $\rho_{MH-ALSH}$ and $\rho_{sign}$ (lower is better) with respect to approximation ratio $c$ for different $\frac{S_0}{M}$. The curves show that asymmetric minhash (solid curves) is noticeably better than ALSH based on signed random projection (dashed curves) in terms of their $\rho$ values, irrespective of the choices of $\frac{S_0}{M}$ or $c$. For clarity, the results are  shown in two panels.}\label{fig:Rho_comp}
\end{figure}

For comparing asymmetric minhash with ALSH for general inner products, we compare $\rho_{MH-ALSH}$ with the ALSH for inner products based on signed random projections. Note that it was shown that $Sign-ALSH$ has better theoretical $\rho$ values as compared to L2-ALSH~\cite{Article:Shrivastava_arXiv14}. Therefore, it suffices to show that asymmetric minhash outperforms signed random projection based ALSH.  Both $\rho_{MH-ALSH}$ and $\rho_{sign}$ can be rewritten in terms of ratio $\frac{S_0}{M}$ as follows. Note that for binary data we have $M = \max_{x \in \mathcal{C}}||x||^2 = V^2$
\begin{align}
\rho_{MH-ALSH} = \frac{\log{\frac{S_0/M}{2 - S_0/M}}}{\log{\frac{cS_0/M}{2 - cS_0/M}}}; \hspace{0.1in}\rho_{Sign} =  \frac{\log{\bigg(1-\frac{1}{\pi}\cos^{-1}\bigg(\frac{S_0}{M}\bigg)\bigg)}}{\log{\bigg(1-\frac{1}{\pi}\cos^{-1} \bigg(\frac{cS_0}{M}\bigg)\bigg)}}
\end{align}
Observe that $M$ is also the upper bound on any inner product. Therefore, we have $0 \le \frac{S_0}{M} \le 1$. We plot the values of $\rho_{MH-ALSH} $ and $\rho_{sign}$ for $\frac{S_0}{M} = \{0.1,\ 0.2,..., 0.8, \ 0.9,\ 0.95 \}$ with $c$. The comparison is summarized in Figure~\ref{fig:Rho_comp}. Note that here we use $\rho_{Sign}$ derived from~\cite{Proc:Bachrach_recsys14,Report:Neyshabur_arXiv14} instead of $\rho_{Sign-ALSH}$ for convenience although the two schemes perform essentially identically.

We can clearly see that irrespective of the choice of threshold $\frac{S_0}{M}$ or the approximation ratio $c$, asymmetric minhash outperforms signed random projection based ALSH in terms of the theoretical $\rho$ values. This is not surprising, because it is known that minwise hashing based methods are often significantly powerful for binary data compared to SRP (or simhash)~\cite{Proc:Shrivastava_AISTATS14}. Therefore ALSH based on minwise hashing outperforms ALSH based on SRP as shown by our theoretical comparisons. Our proposal thus leads to an algorithmic improvement over state-of-the-art hashing techniques for retrieving binary inner products.

\section{Evaluations}

In this section, we compare the different hashing schemes on the actual task of retrieving top-ranked elements based on set Jaccard containment. The experiments are divided into two parts. In the first part, we show how the ranking based on various hash functions correlate with the ordering of  Jaccard containment. In the second part, we perform the actual LSH based bucketing experiment for retrieving top-ranked elements and compare the computational saving obtained by various hashing algorithms.

\subsection{ Datasets}

 We chose four publicly available high dimensional sparse datasets: {\em EP2006}\footnote{We downloaded EP2006 from  LIBSVM website. The original name is ``E2006LOG1P'' and we re-name it to ``EP2006''.},  {\em MNIST}, {\em NEWS20}, and {\em NYTIMES}. Except MNIST, the other three are  high dimensional binary ``BoW" representation of the corresponding text corpus. MNIST is an image dataset consisting of 784 pixel image of handwritten digits. Binarized versions of MNIST are commonly used in literature. The pixel values in MNIST were binarized to 0 or 1 values.  For each of the four datasets, we generate two partitions. The bigger partition was used to create hash tables and is referred as the {\bf training partition}. The small partition which we call the {\bf query partition} is used for querying.  The statistics of these datasets are summarized in Table~\ref{tab_data}. The datasets cover a wide spectrum of sparsity and dimensionality.

\begin{table}[h!]
\vspace{-0.1in}
\caption{Datasets}
\begin{center}{\small
{\begin{tabular}{l r r r c}
\hline \hline
Dataset        &\# Query    &\# Train  &\# Dim & nonzeros (mean $\pm$ std) \\
\hline
EP2006      &2,000 &17,395 & 4,272,227 & 6072 $\pm$ 3208\\
MNIST      &2,000 &68,000 &784  & 150 $\pm$ 41 \\
NEWS20     &2,000  &18,000 &1,355,191 &454 $\pm$ 654 \\
NYTIMES    &2,000 &100,000 &102,660 &232 $\pm$ 114
\\\hline\hline
\end{tabular}}
}
\end{center}\label{tab_data}\vspace{-0.2in}
\end{table}

\subsection{Competing Hash Functions}
We consider the following hash functions for evaluations:
\begin{enumerate}
\item {\bf Asymmetric minwise hashing (Proposed):} This is our proposal, the asymmetric minhash described in Section~\ref{sec:Amin}.

\item {\bf Traditional minwise hashing (MinHash):} This is the usual minwise hashing, the  popular heuristic described in Section~\ref{sec:minhash}. This is a symmetric hash function, we use $h_{\pi}$ as define in Eq.(\ref{eq:min}) for both query and the training set.

\item {\bf L2 based Asymmetric LSH for Inner products (L2-ALSH):} This is the asymmetric LSH of~\cite{Proc:Shrivastava_NIPS14} for general inner products based on LSH for L2 distance.

\item {\bf SRP based Asymmetric LSH for Inner Products (Sign-ALSH):} This is the asymmetric hash function of~\cite{Article:Shrivastava_arXiv14} for general inner products based on SRP.
\end{enumerate}

\subsection{Ranking Experiment: Hash Quality Evaluations}

We are interested in knowing, how the orderings under different competing hash functions correlate with the ordering of the underlying similarity measure which in this case is the Jaccard containment.  For this task, given a query $q$ vector, we compute the top-100 gold standard elements from the training set based on the Jaccard containment $\frac{a}{f_q}$. Note that this is the same as the top-100 elements based on binary inner products.  Give a query $q$, we compute $K$ different hash codes of the vector $q$ and all the vectors in the training set.  We then compute the number of times the hash values of a vector $x$ in the training set matches (or collides) with the hash values of query $q$ defined by
\begin{equation}
Matches_x = \sum_{t=1}^{K} {\bf 1}(h_t(q) = h_t(x)),
\end{equation}
where ${\bf 1}$ is the indicator function.  $t$ subscript is used to distinguish independent draws of the underlying hash function. Based on $Matches_x$ we rank all elements in the training set. This procedure generates a sorted list for every query for every hash function. For asymmetric hash functions, in computing total collisions, on the query vector we use the corresponding $Q$ function (query transformation)  followed by underlying hash function, while for elements in the training set we use the $P$ function (preprocessing transformation) followed by the corresponding hash function.

We compute the precision and the recall of the top-100 gold standard elements in the ranked list generated by different hash functions.  To compute precision and recall, we start at the top of the ranked item list and walk down in order, suppose we are at the $p^{th}$ ranked element, we check if this element belongs to the gold standard top-100 list. If it is one of the top 100 gold standard elements, then we increment the count of \emph{relevant seen} by 1, else we move to $p+1$. By $p^{th}$ step, we have already seen $p$ elements, so the \emph{total elements seen} is $p$. The precision and recall at that point is then computed as:
\begin{align}
 Precision = \frac{\text{relevant seen}}{p}, \hspace{0.2in}
Recall = \frac{\text{relevant seen}}{100}
\end{align}
It is important to balance both.  Methodology which obtains higher precision at a given recall is superior. Higher precision indicates higher ranking of the relevant items.
We finally average these values of precision and recall over all elements in the query set. The results for $K\in \{32, \ 64, \ 128\}$  are summarized in Figure~\ref{fig:hashquality}.

We can clearly see, that the proposed hashing scheme always achieves better, often significantly, precision at any given recall compared to other hash functions. The two ALSH schemes are usually always better than traditional minwise hashing. This confirms that fact that ranking based on collisions under minwise hashing can be different from the rankings under Jaccard containment or inner products. This is expected, because minwise hashing in addition penalizes the number of nonzeros leading to a ranking very different from the ranking of inner products. Sign-ALSH usually performs better than L2-LSH, this is in line with the results obtained in~\cite{Article:Shrivastava_arXiv14}.\\

\begin{figure*}[!ht]
\begin{center}

\mbox{
\includegraphics[width=2.2in]{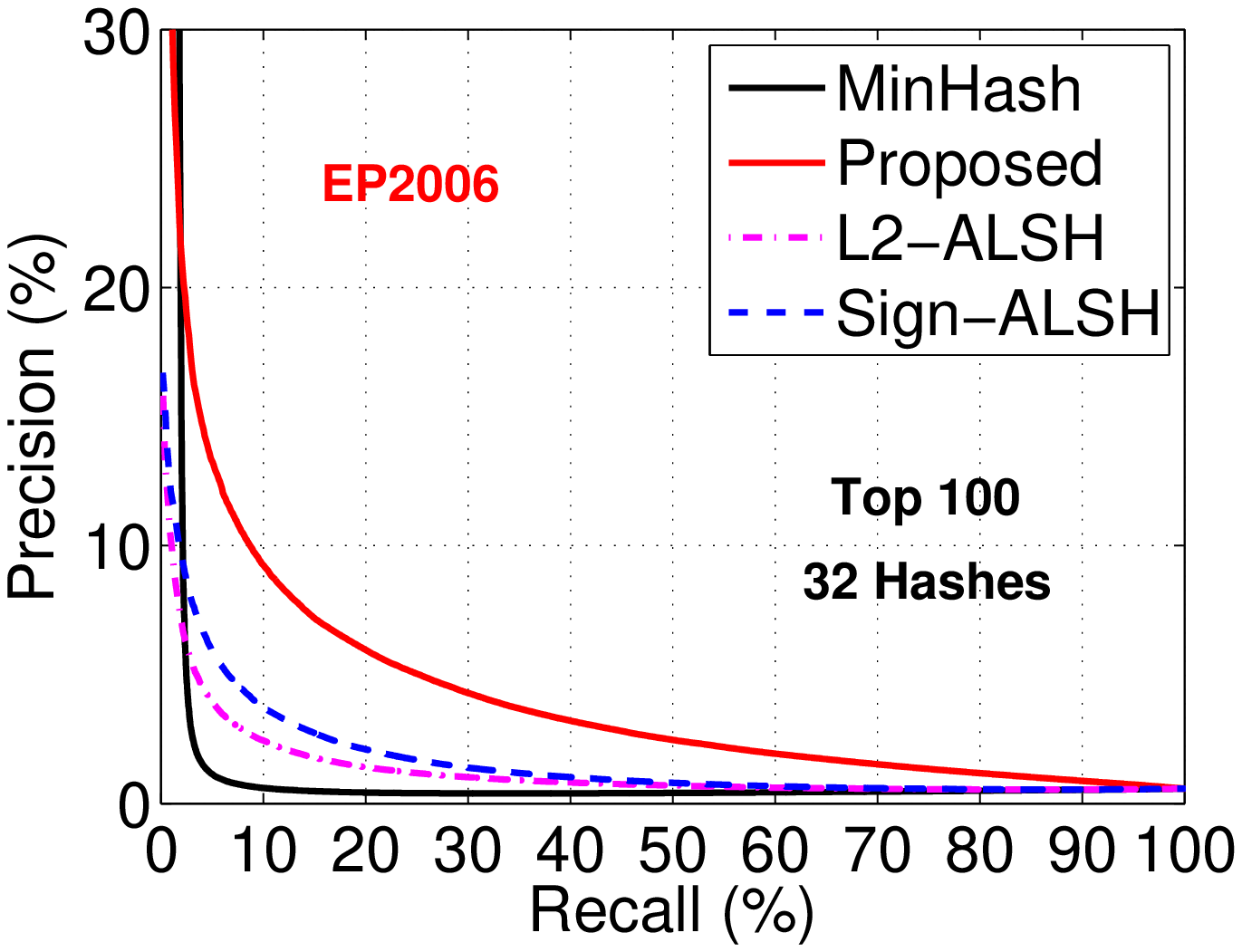}\hspace{-0.13in}
\includegraphics[width=2.2in]{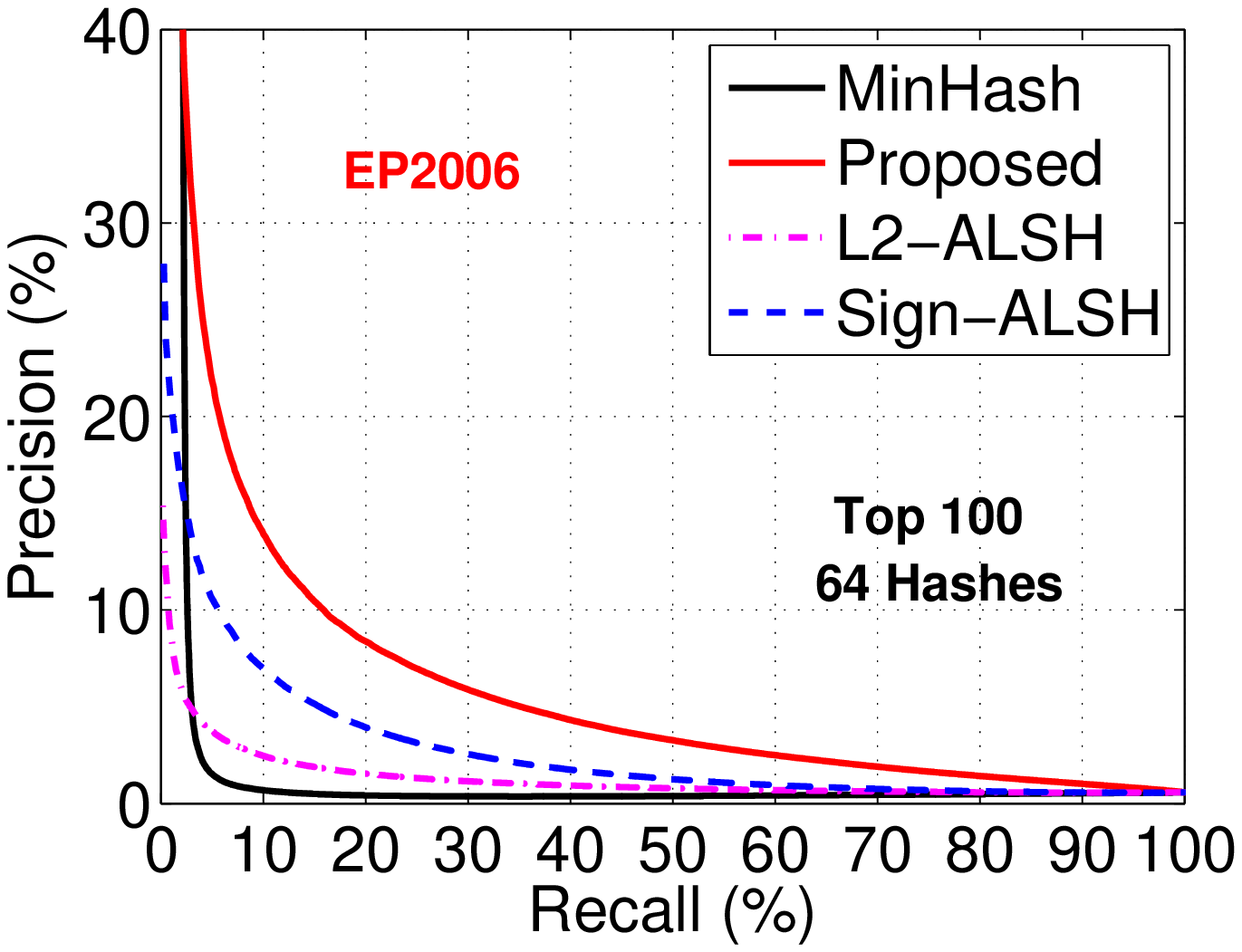}\hspace{-0.13in}
\includegraphics[width=2.2in]{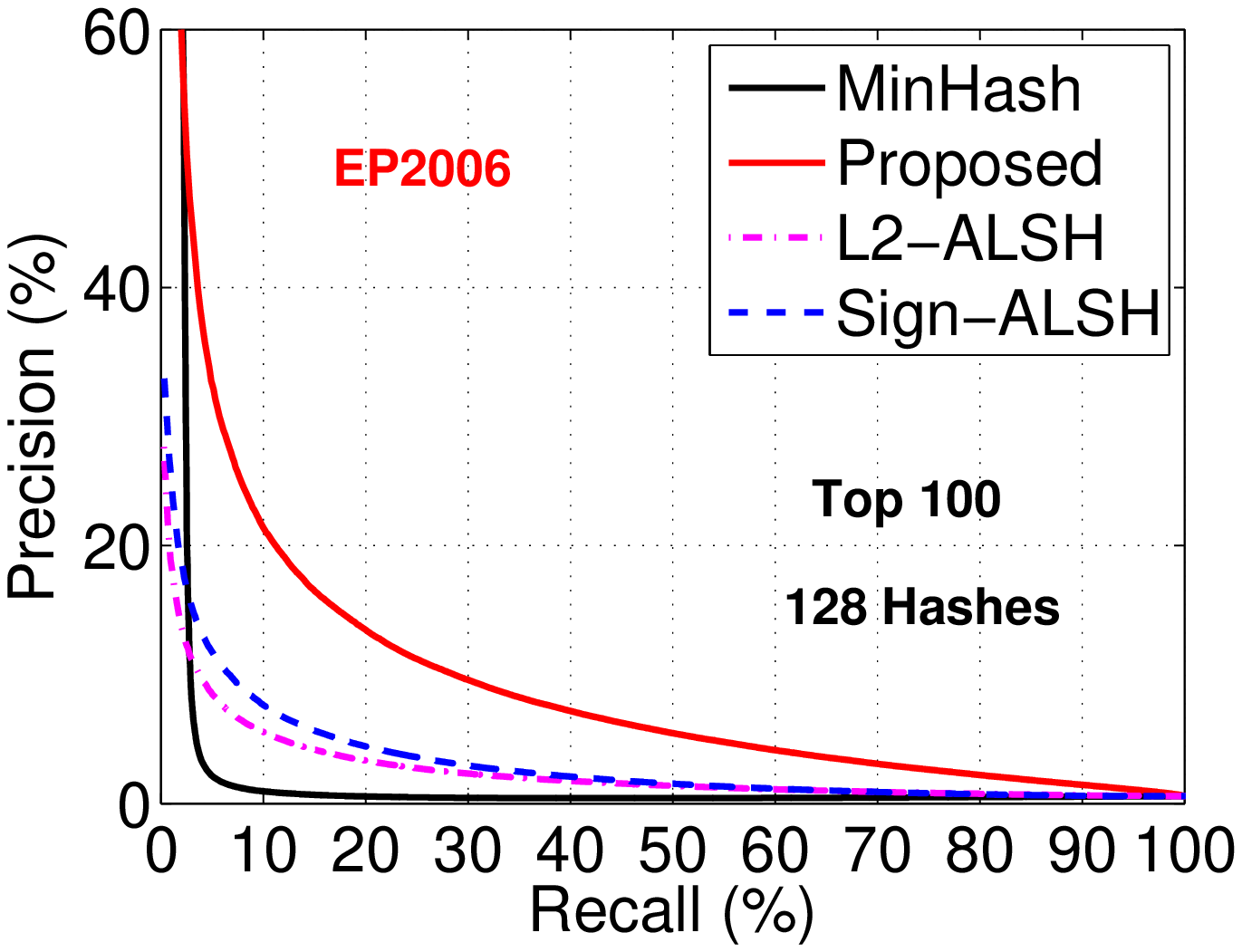}}
\mbox{
\includegraphics[width=2.2in]{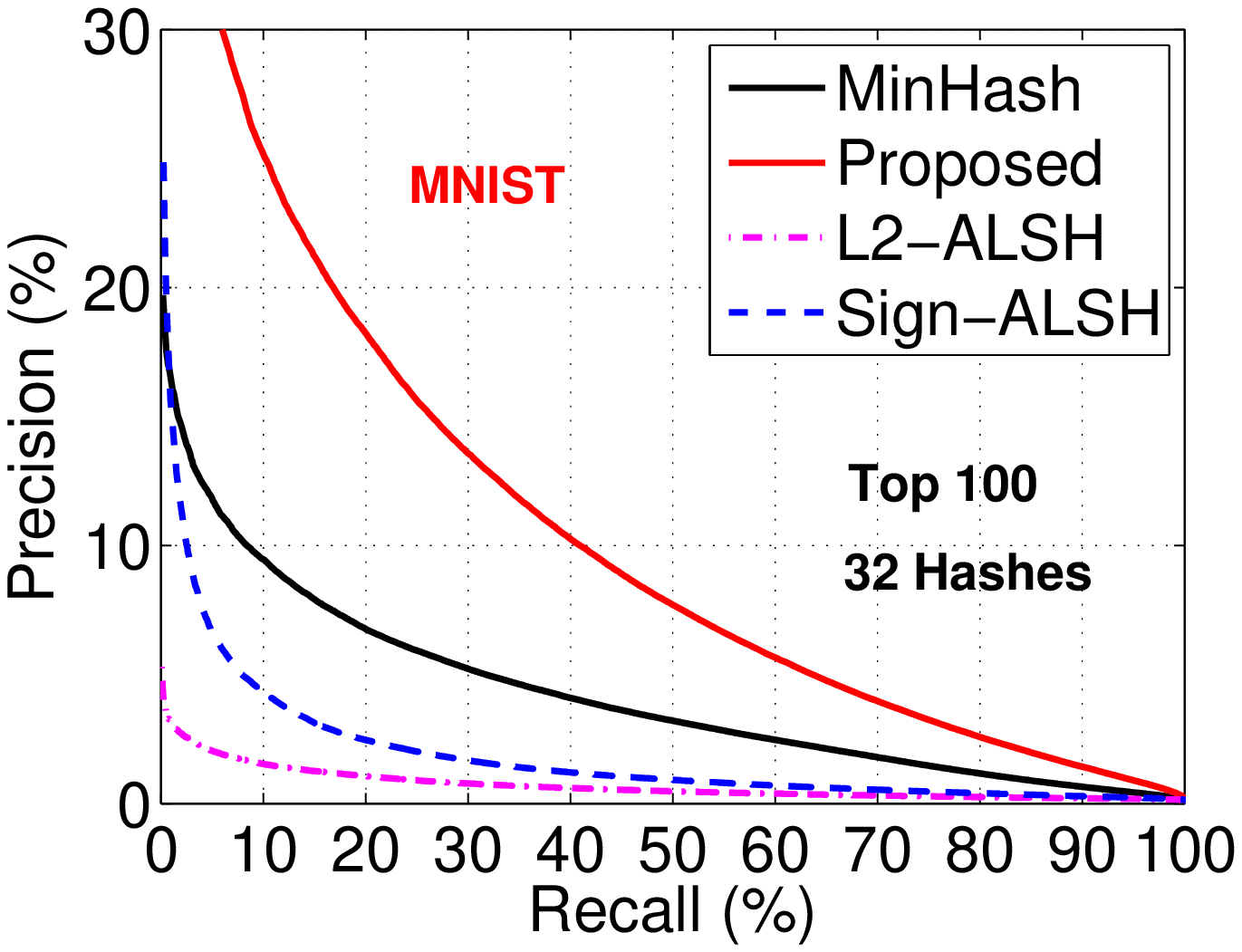}\hspace{-0.13in}
\includegraphics[width=2.2in]{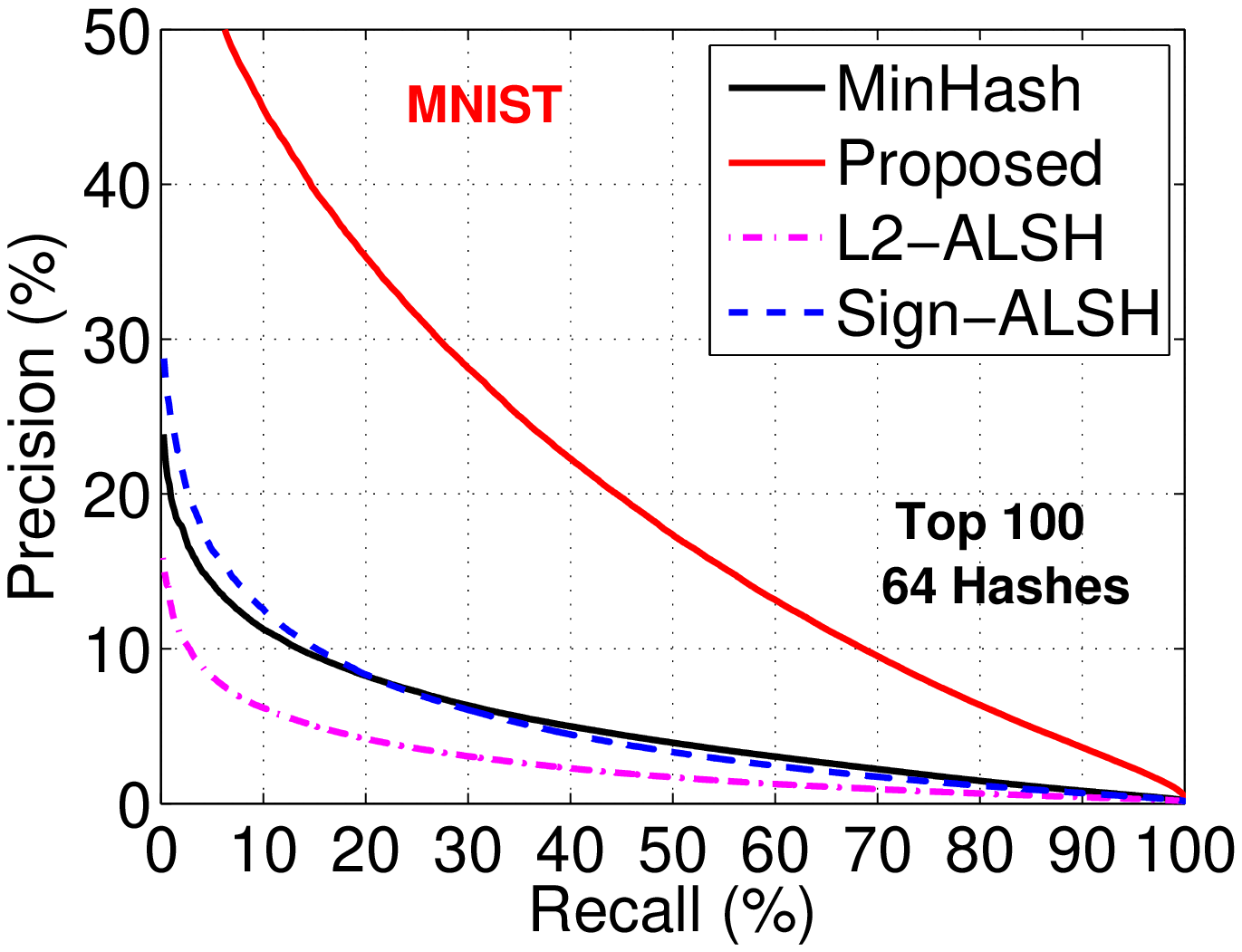}\hspace{-0.13in}
\includegraphics[width=2.2in]{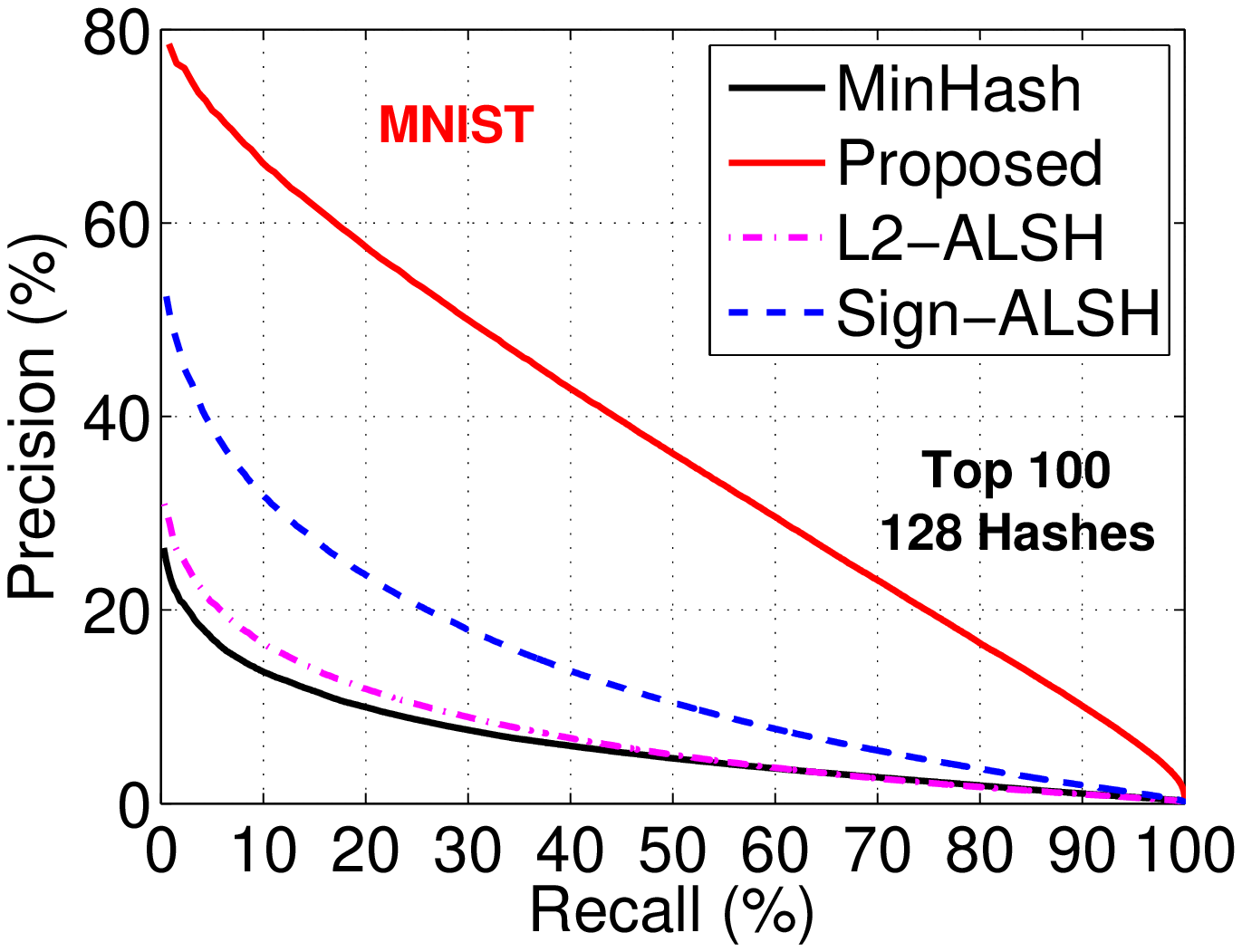}}
\mbox{
\includegraphics[width=2.2in]{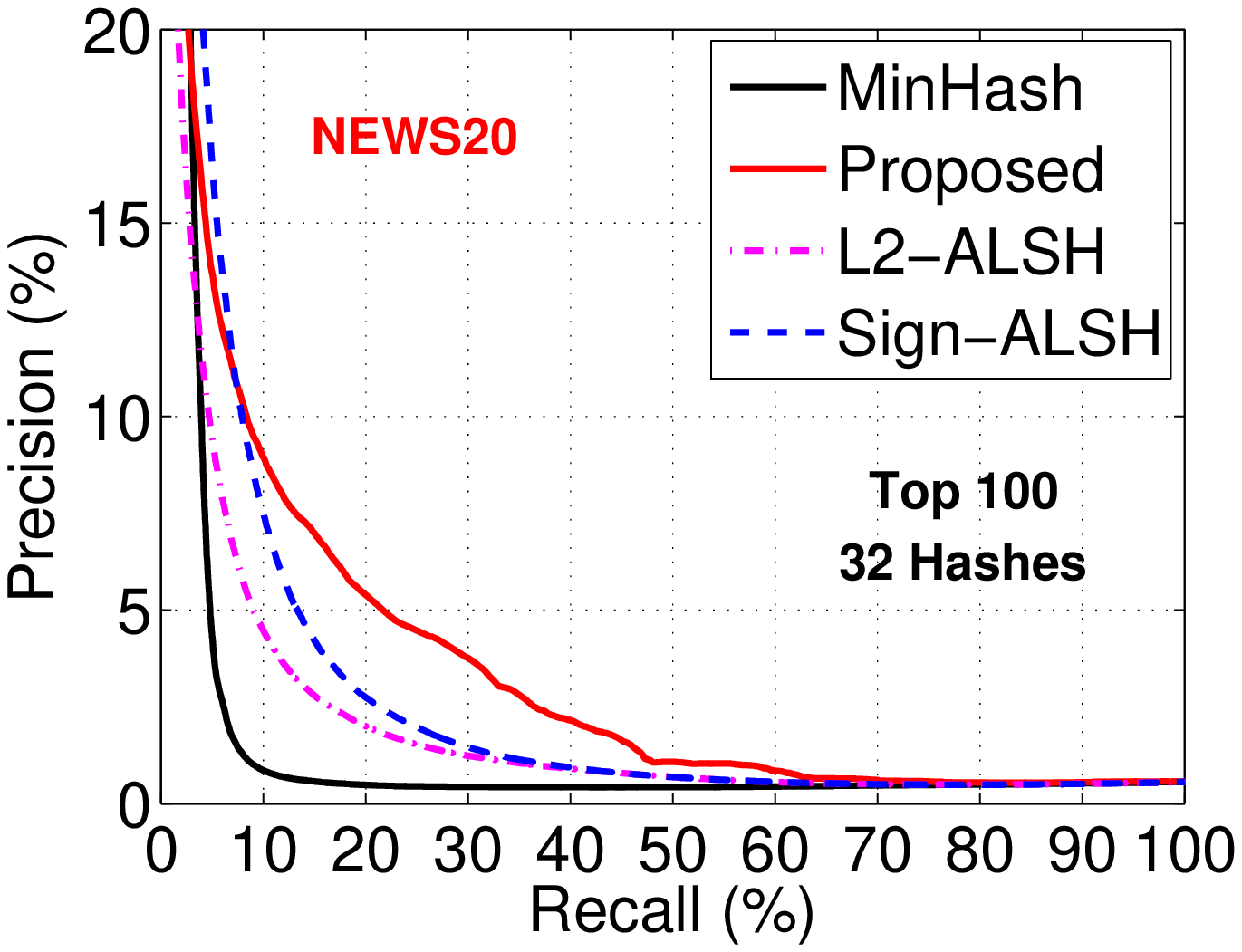}\hspace{-0.13in}
\includegraphics[width=2.2in]{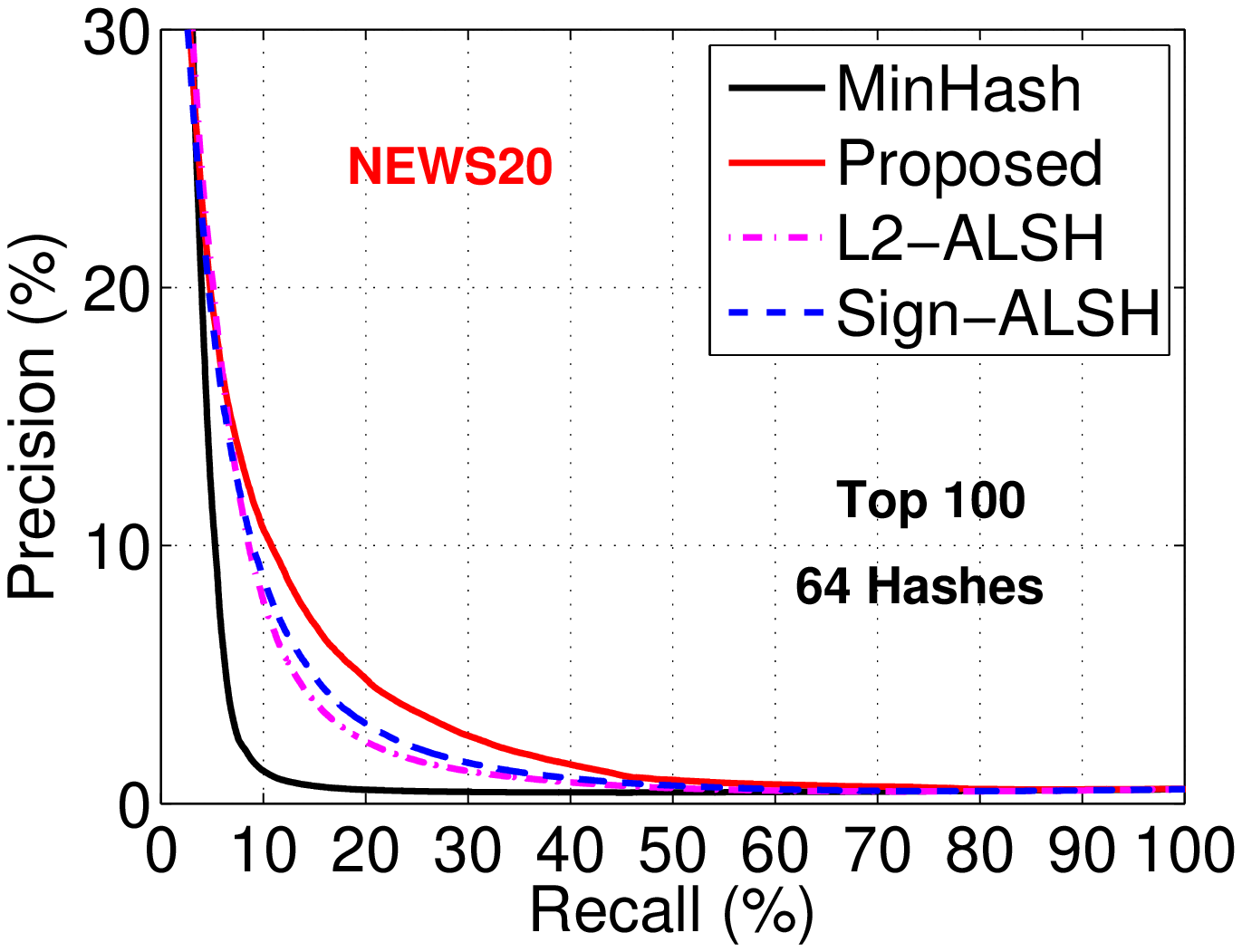}\hspace{-0.13in}
\includegraphics[width=2.2in]{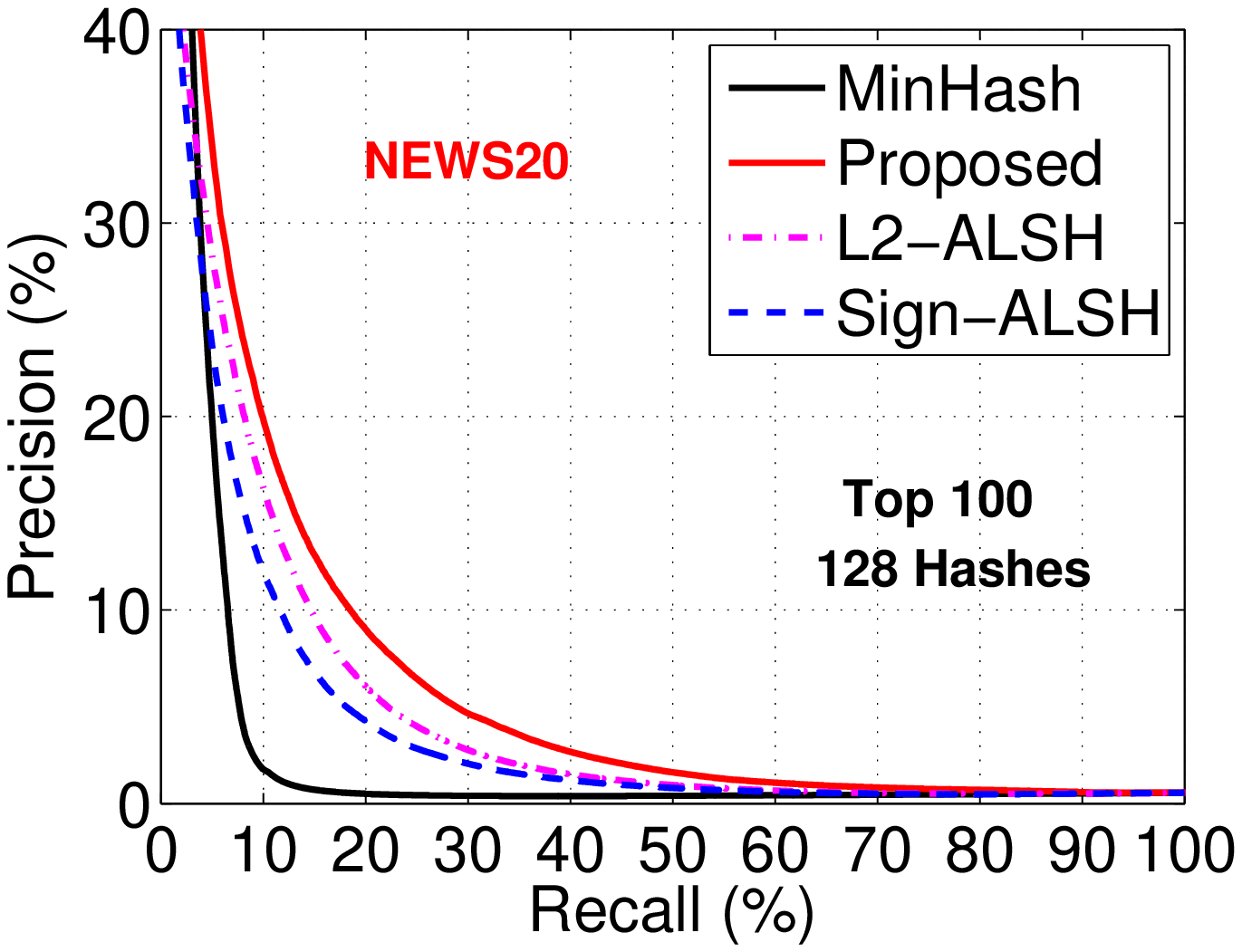}}

\mbox{
\includegraphics[width=2.2in]{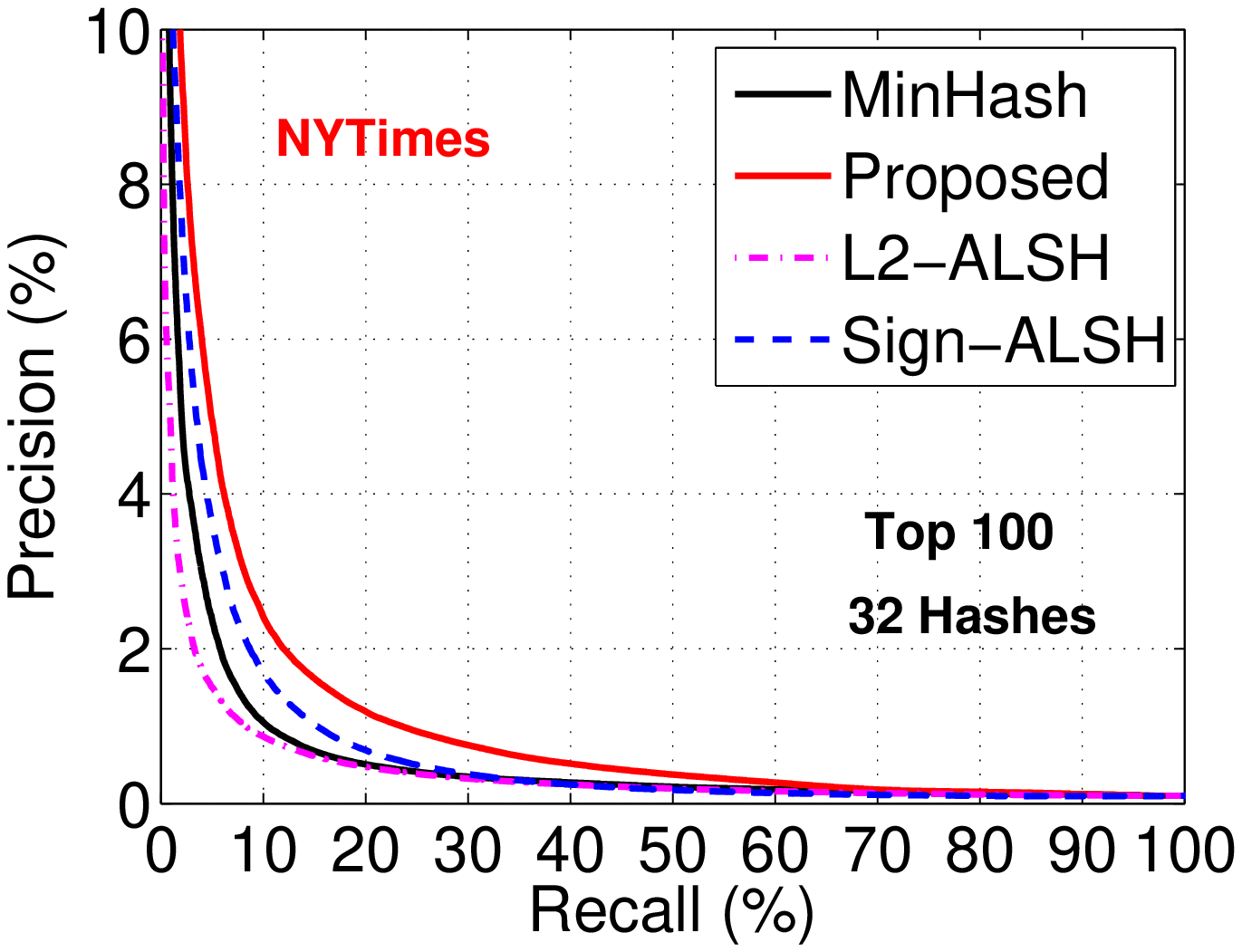}\hspace{-0.13in}
\includegraphics[width=2.2in]{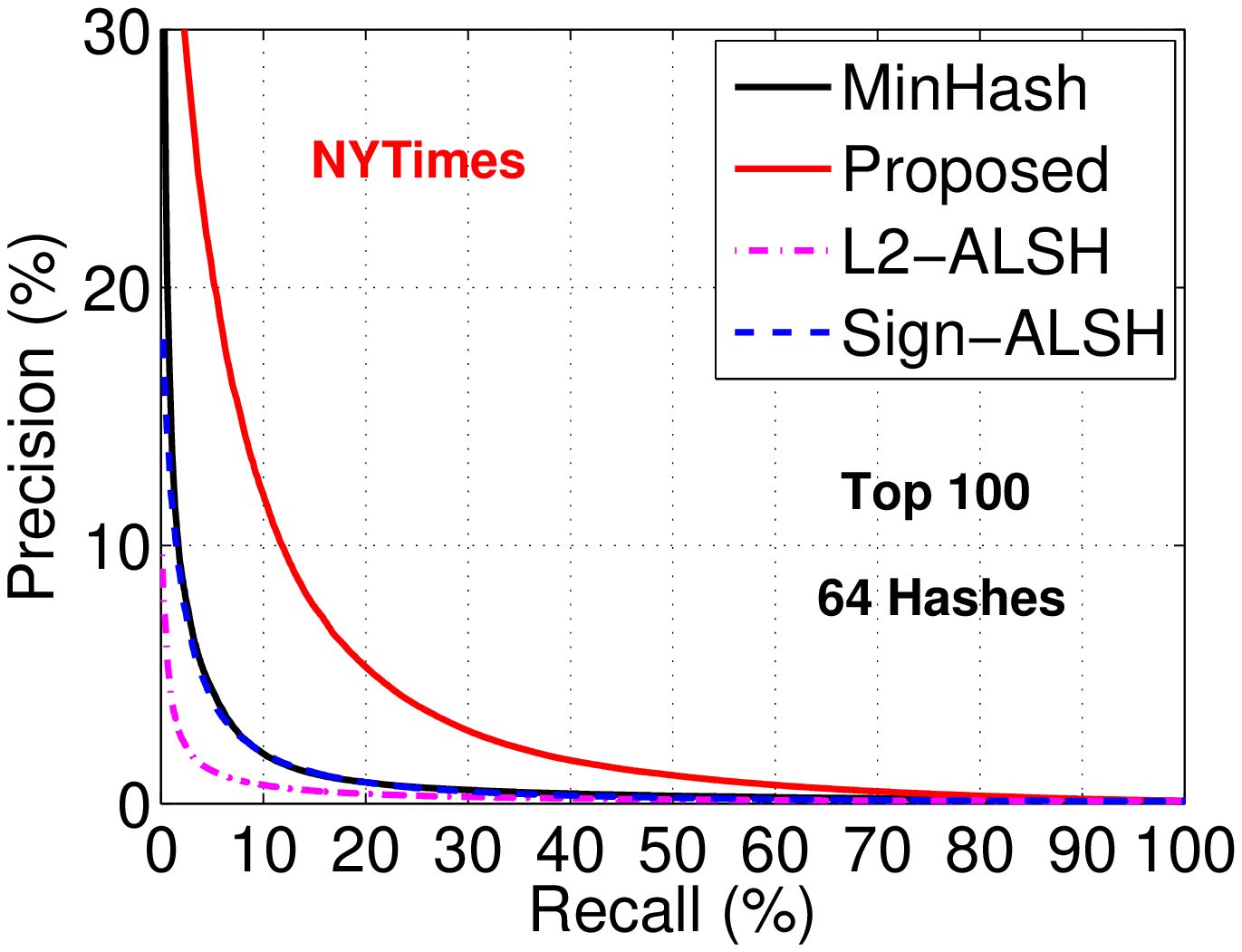}\hspace{-0.13in}
\includegraphics[width=2.2in]{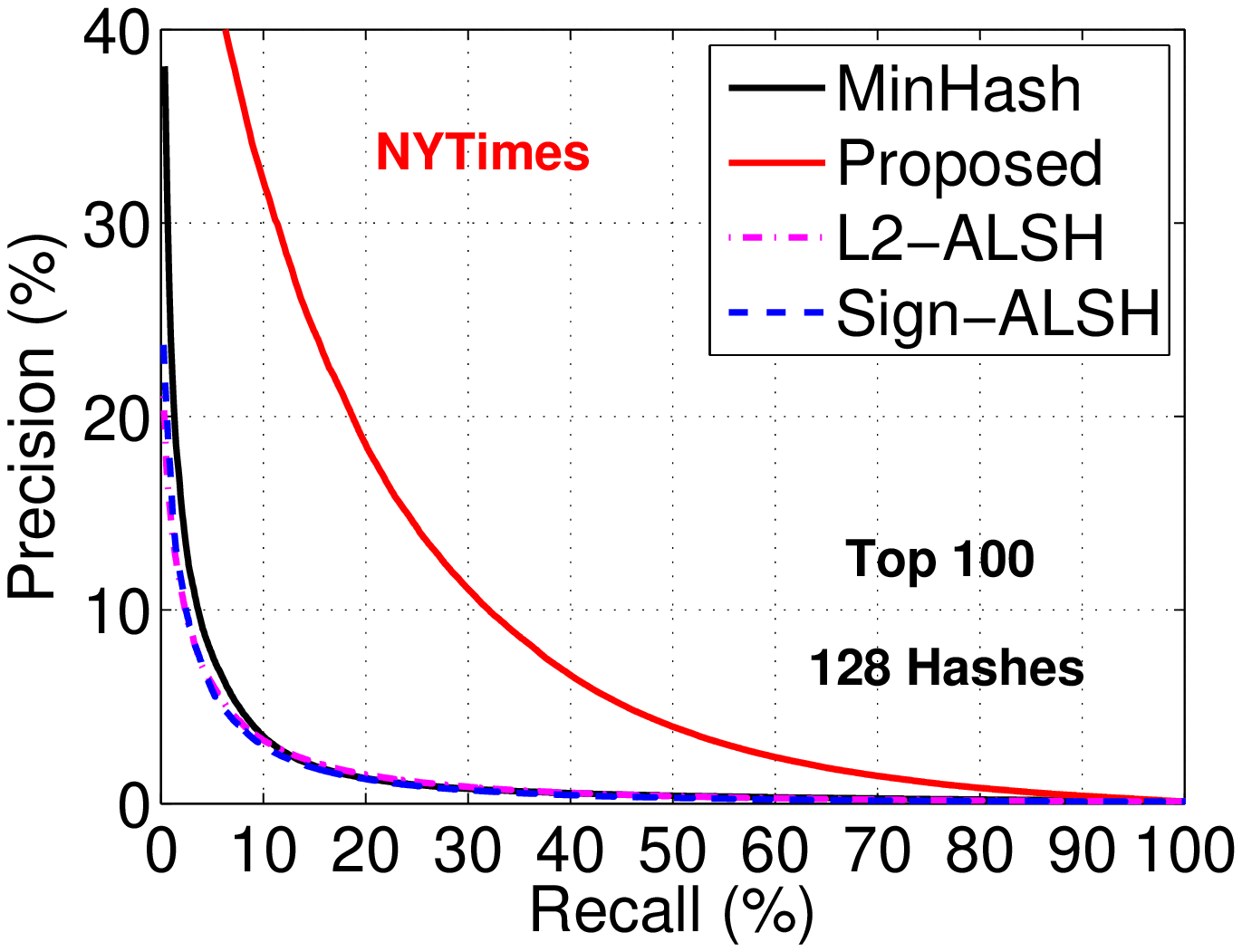}}

\end{center}
\vspace{-0.2in}
\caption{\textbf{Ranking Experiments}. \ Precision Vs Recall curves for retrieving top-100 items, for different hashing schemes on 4 chosen datasets. The precision and the recall were computed based on the rankings obtained by different hash functions using 32, 64 and 128 independent hash evaluations. Higher precision at a given recall is better.}\label{fig:hashquality}
\end{figure*}

\newpage\clearpage
It should be noted that ranking experiments only validate the monotonicity of the collision probability. Although, better ranking is definitely a very good indicator of good hash function, it does not always mean that we will achieve faster sub-linear LSH algorithm. For bucketing the probability sensitivity around a particular threshold is the most important factor, see~\cite{Book:Raj_Ullman} for more details. What matters is the \textbf{gap} between the collision probability of good and the bad points. In the next subsection, we compare these schemes on the actual task of near neighbor retrieval with Jaccard containment.

\subsection{Bucketing Experiment: Computational Savings in Near Neighbor Retrieval}

In this section, we evaluate the four hashing schemes on the standard $(K,L)$-parameterized bucketing algorithm~\cite{Report:E2LSH} for sub-linear time retrieval of near neighbors based on Jaccard containment.  In $(K,L)$-parameterized LSH algorithm, we generate $L$ different meta-hash functions. Each of these meta-hash functions is formed by concatenating $K$ different hash values as
\begin{equation}
\label{eq:bucket}
B_j(x)  = [h_{j1}(x);h_{j2}(x);...;h_{{jK}}(x)],
\end{equation}
 where $h_{ij}, i \in \{1,2,...,K \}$ and $j \in  \{1,2,...,L \}$, are $KL$ different independent evaluations of the hash function under consideration. Different competing scheme uses its own underlying randomized hash function $h$.\\

In general, the $(K,L)$-parameterized LSH works in two phases:
\begin{enumerate}[i)]
\item {\bf Preprocessing Phase:} We construct $L$ hash tables from the data by storing element $x$, in the training set, at location $B_j(P(x))$ in the hash-table $j$. Note that for vanilla minhash which is a symmetric hashing scheme $P(x) = x$. For other asymmetric schemes, we use their corresponding $P$ functions. Preprocessing is a one time operation, once the hash tables are created they are fixed.
\item {\bf Query Phase:} Given a query $q$, we report the union of all the points in the buckets $B_j(Q(q))$ $\forall j \in \{1,2,...,L\}$, where the union is over $L$ hash tables. Again here $Q$ is the corresponding $Q$ function of the asymmetric hashing scheme, for  minhash $Q(x) = x$.
\end{enumerate}

Typically, the performance of a bucketing algorithm is  sensitive to the choice of  parameters $K$ and $L$. Ideally, to find best $K$ and $L$, we need to know the operating threshold $S_0$ and the approximation ratio $c$ in advance. Unfortunately, the data and the queries are very diverse and therefore for retrieving top-ranked near neighbors there are no common fixed threshold $S_0$ and approximation ratio $c$ that work for all the  queries.\\

Our objective is to compare the four hashing schemes and minimize the effect of $K$ and $L$, if any, on the evaluations. This is achieved by finding best $K$ and $L$ at every recall level.  We run the bucketing experiment for all combinations of $K \in \{1,2,3,...40\}$ and $L \in \{1,2,3,...,400\}$ for all the four hash functions independently. These choices include the recommended optimal combinations at various thresholds.   We then compute, for every $K$ and $L$, the mean recall of Top-$T$  pairs and the mean number of points reported, per query, to achieve that recall. The best $K$ and $L$ at every recall level is chosen independently for different $T$s. The plot of the mean fraction of points scanned with respect to the recall of top-$T$ gold standard near neighbors, where $T \in \{5, 10, 20, 50\}$,  is summarized in Figure~\ref{fig:topk}.

\begin{figure*}[!ht]
\begin{center}

\mbox{
\includegraphics[width=1.8in]{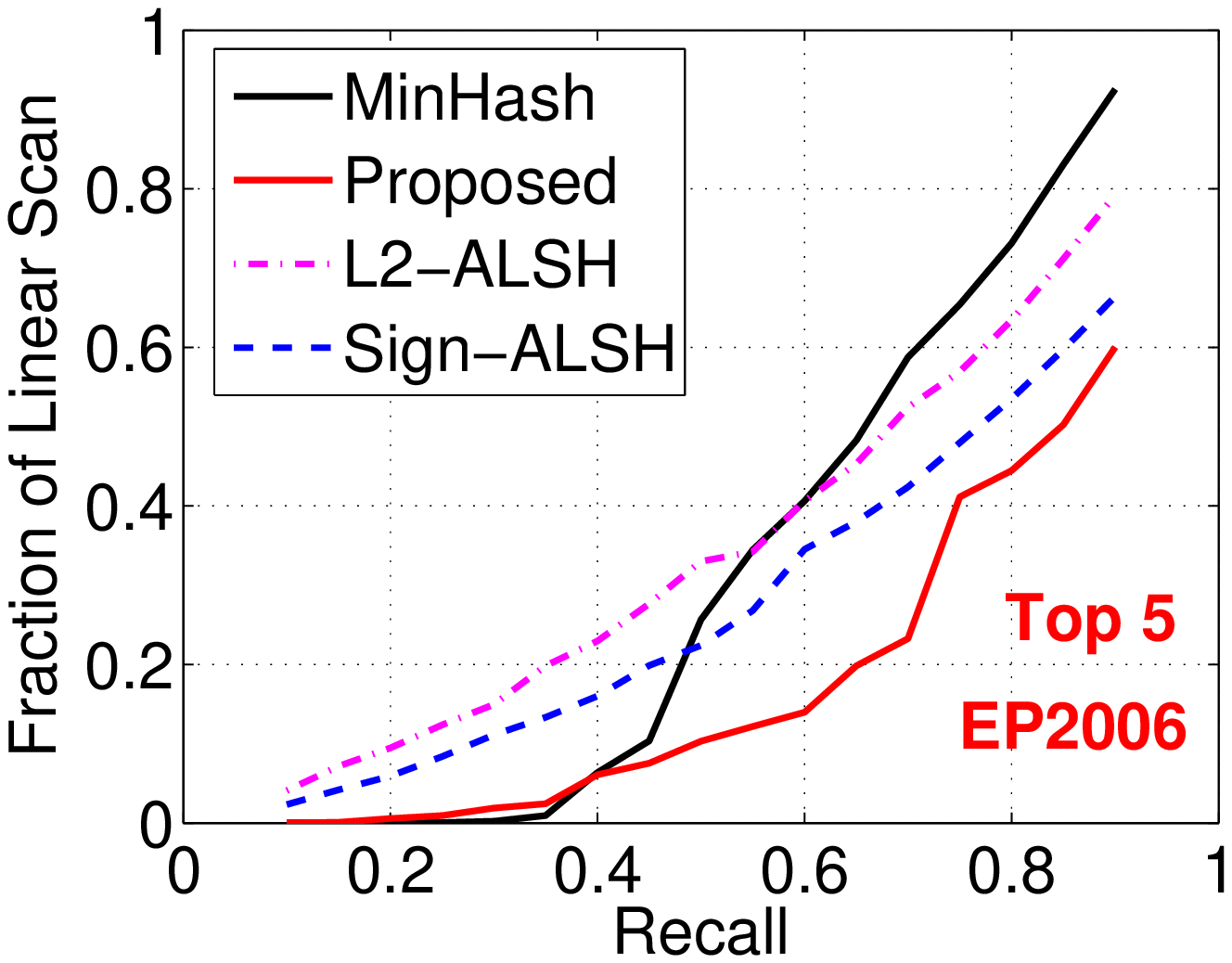}\hspace{-0.13in}
\includegraphics[width=1.8in]{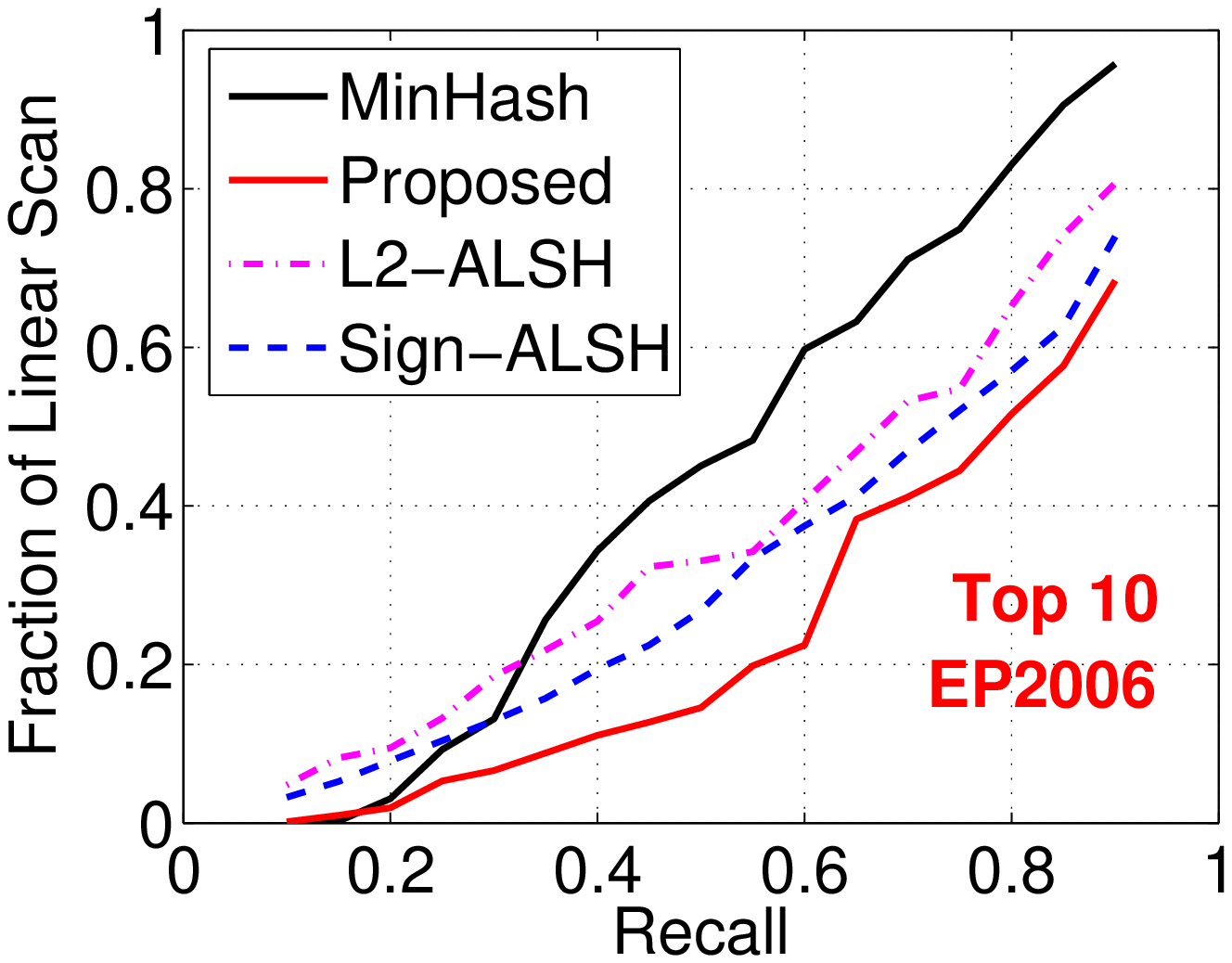}\hspace{-0.13in}
\includegraphics[width=1.8in]{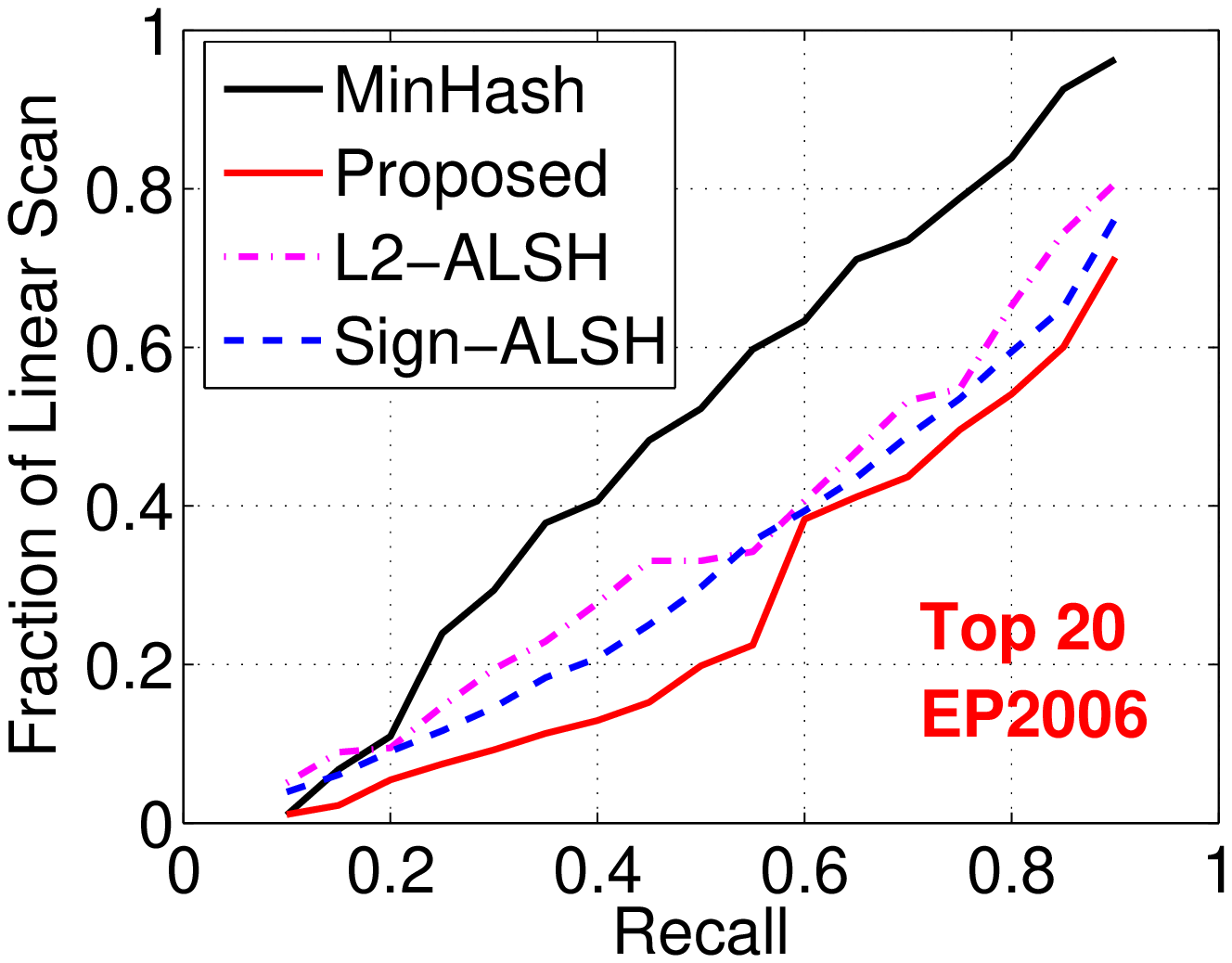}\hspace{-0.13in}
\includegraphics[width=1.8in]{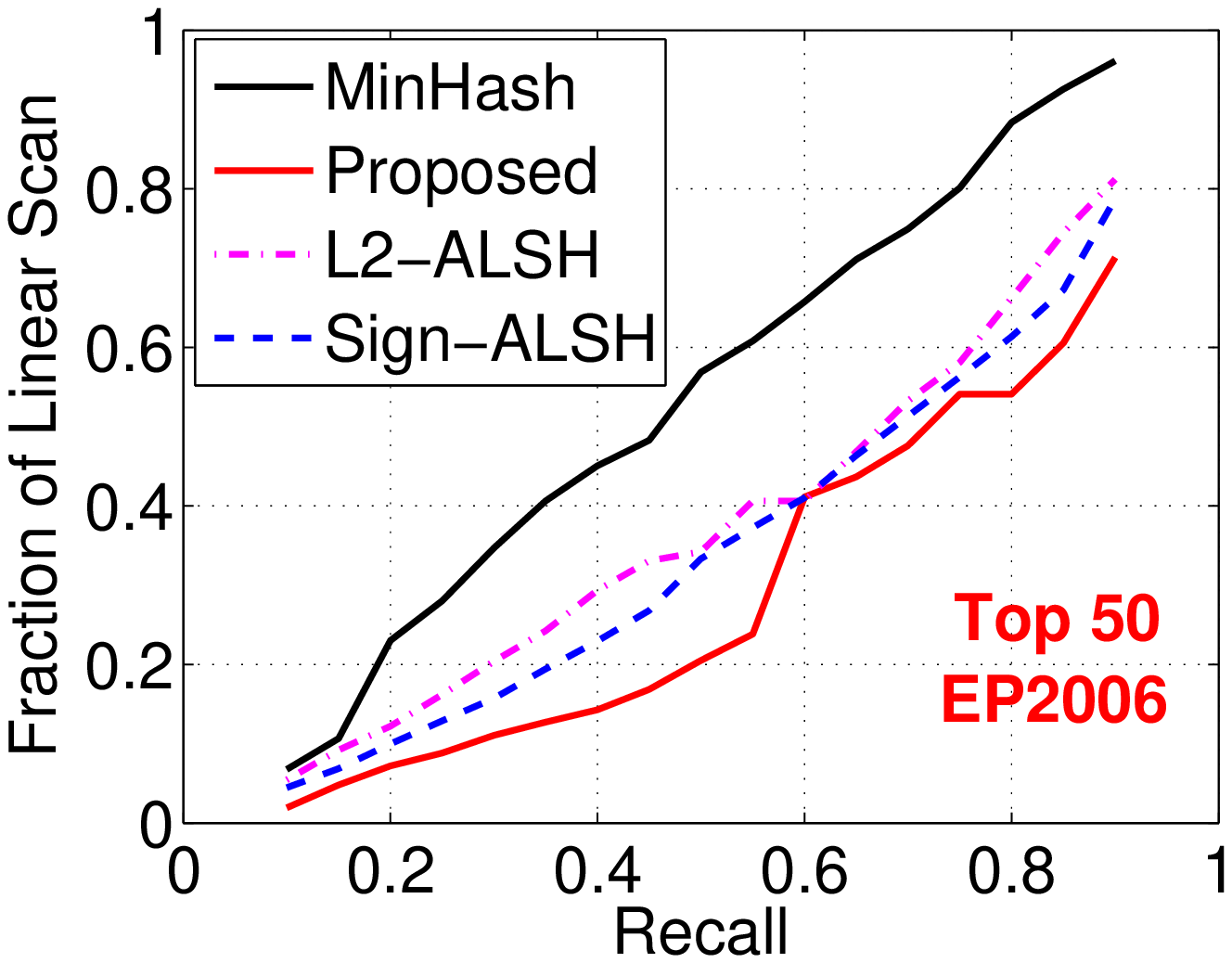}
}
\mbox{
\includegraphics[width=1.8in]{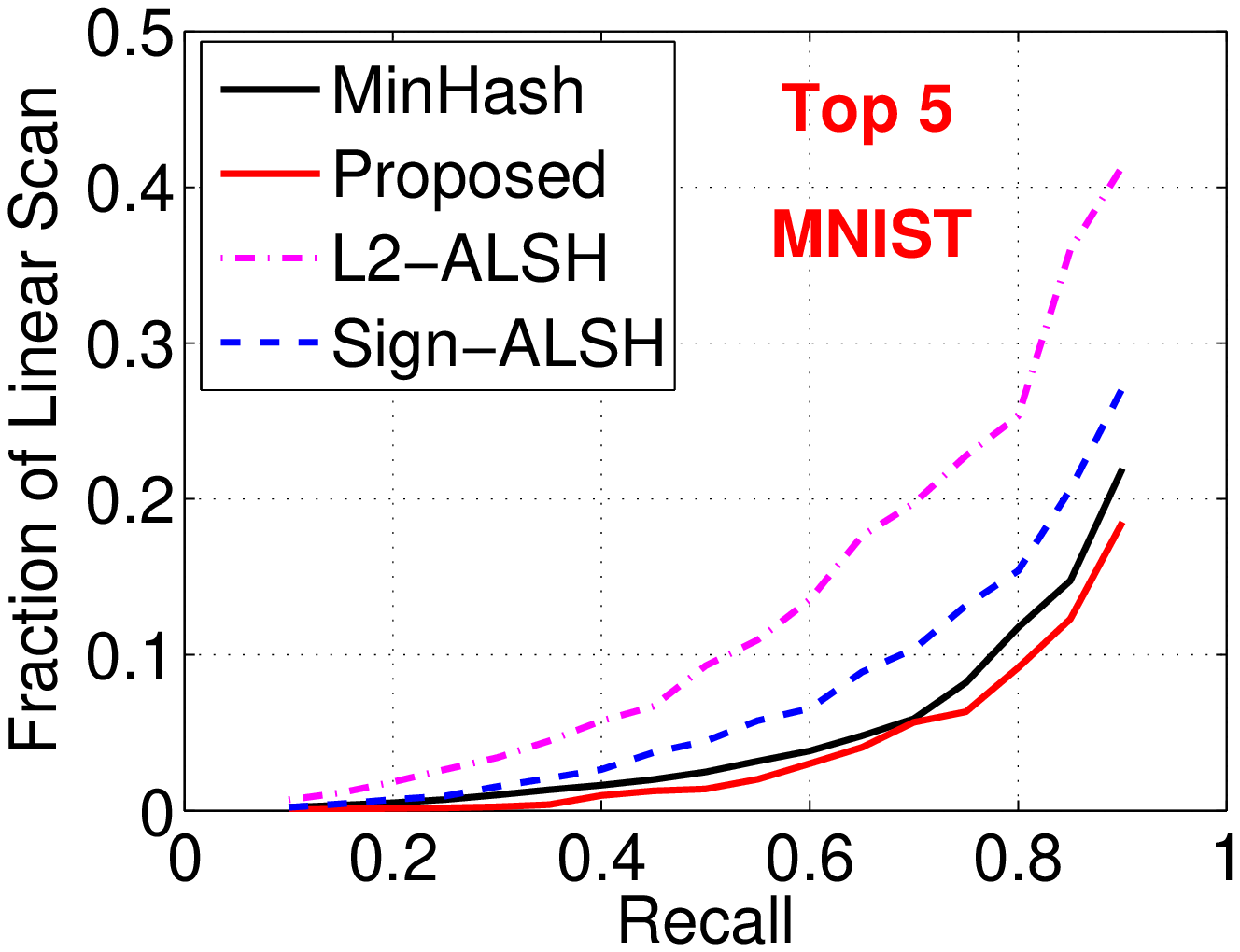}\hspace{-0.13in}
\includegraphics[width=1.8in]{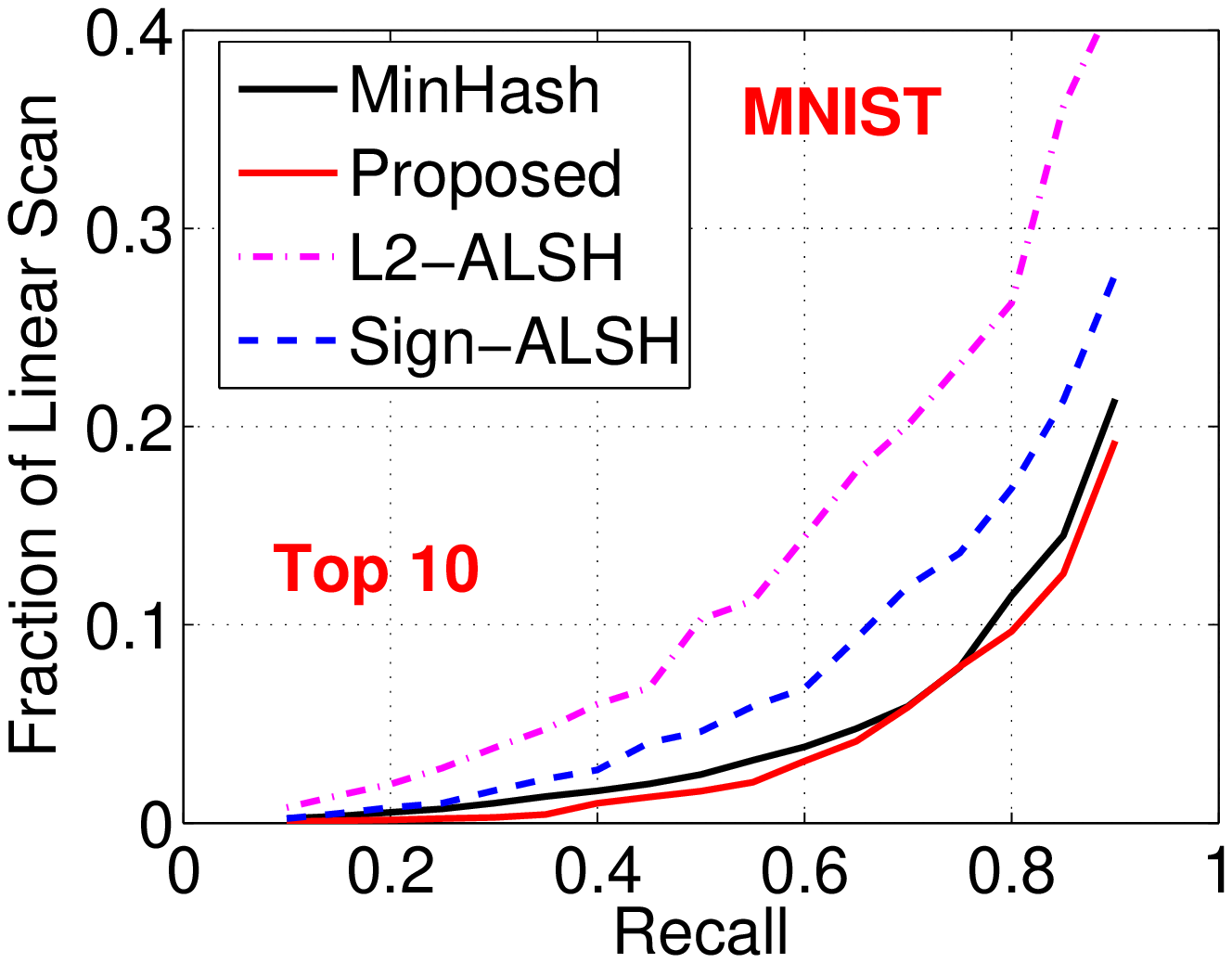}\hspace{-0.13in}
\includegraphics[width=1.8in]{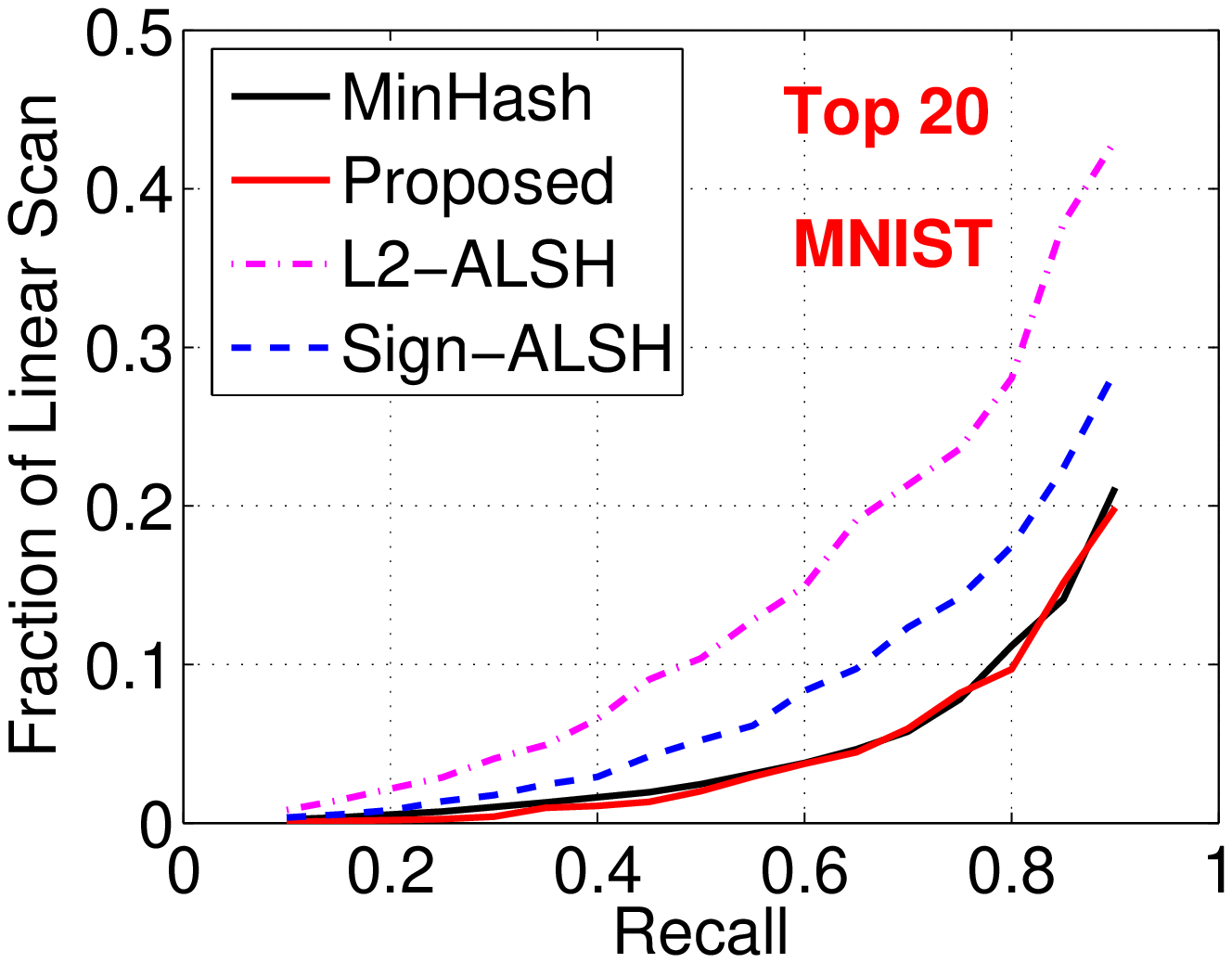}\hspace{-0.13in}
\includegraphics[width=1.8in]{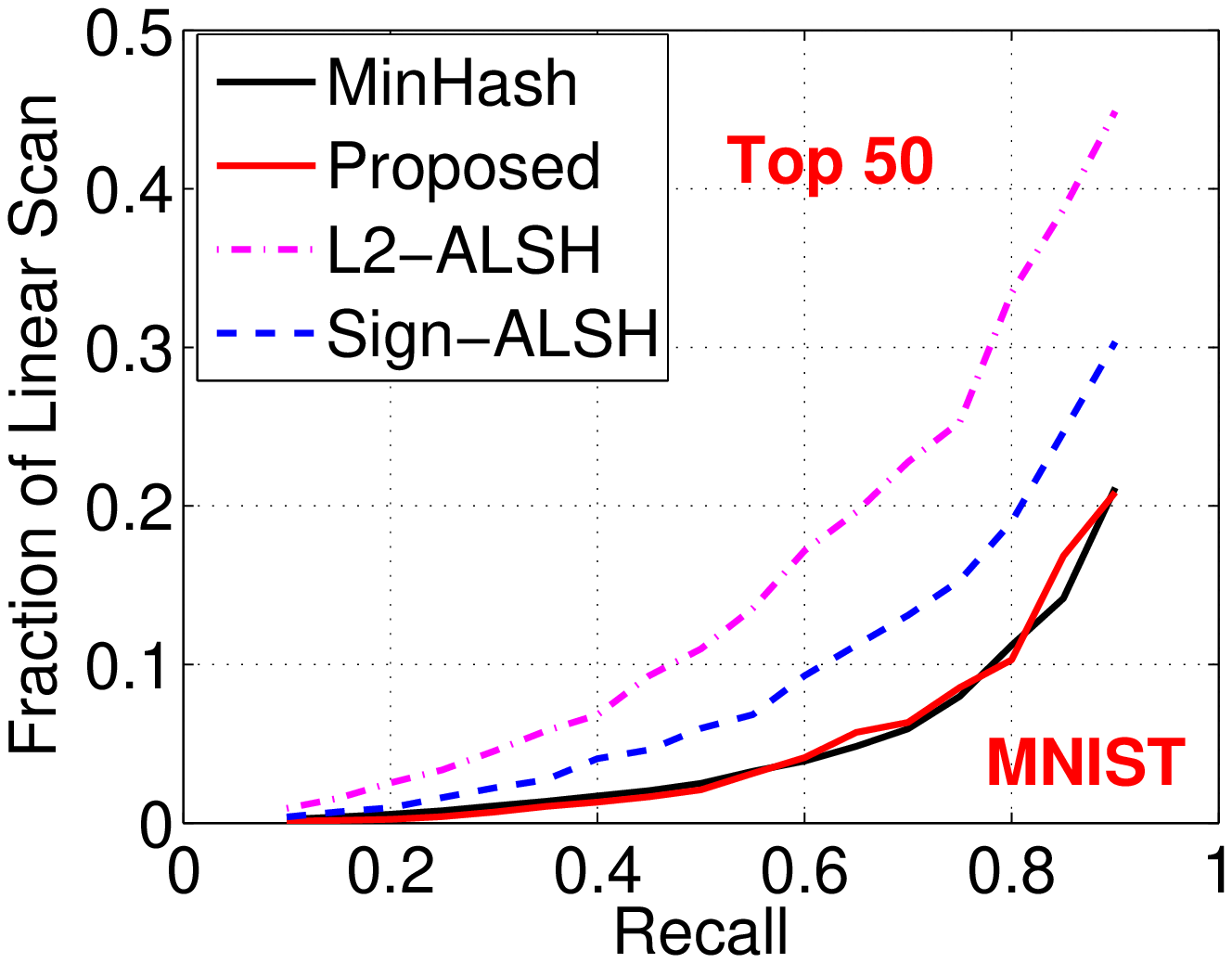}
}
\mbox{
\includegraphics[width=1.8in]{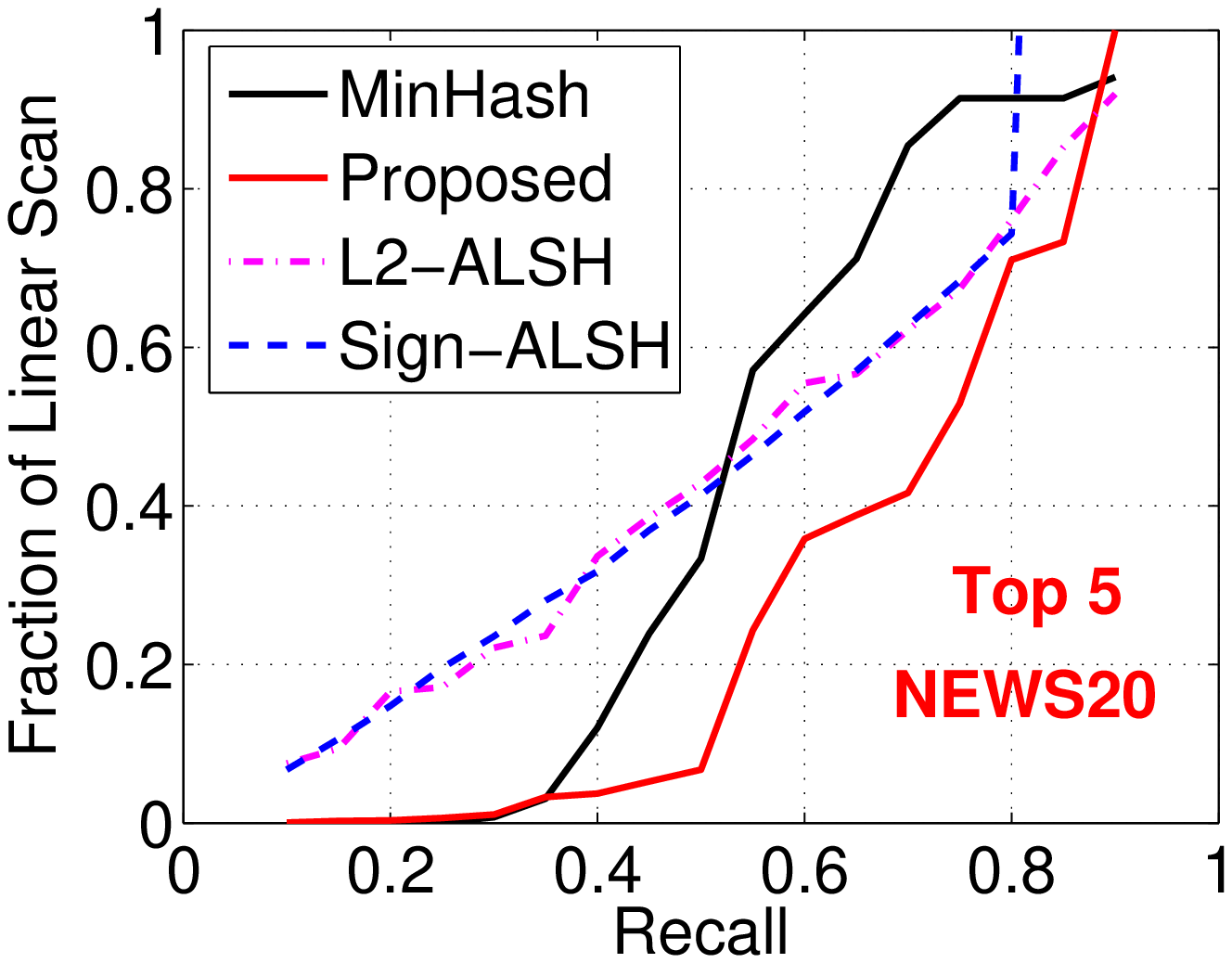}\hspace{-0.13in}
\includegraphics[width=1.8in]{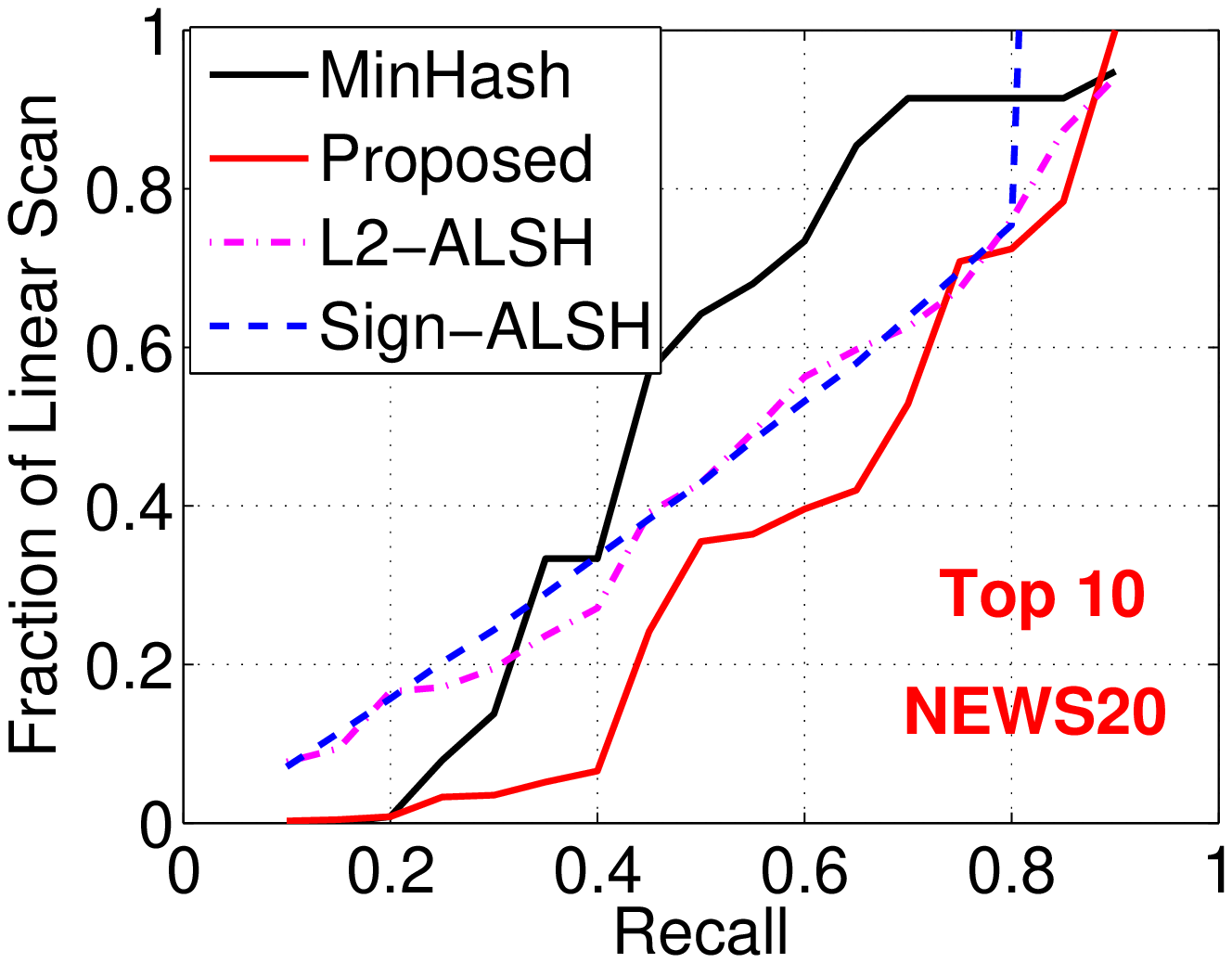}\hspace{-0.13in}
\includegraphics[width=1.8in]{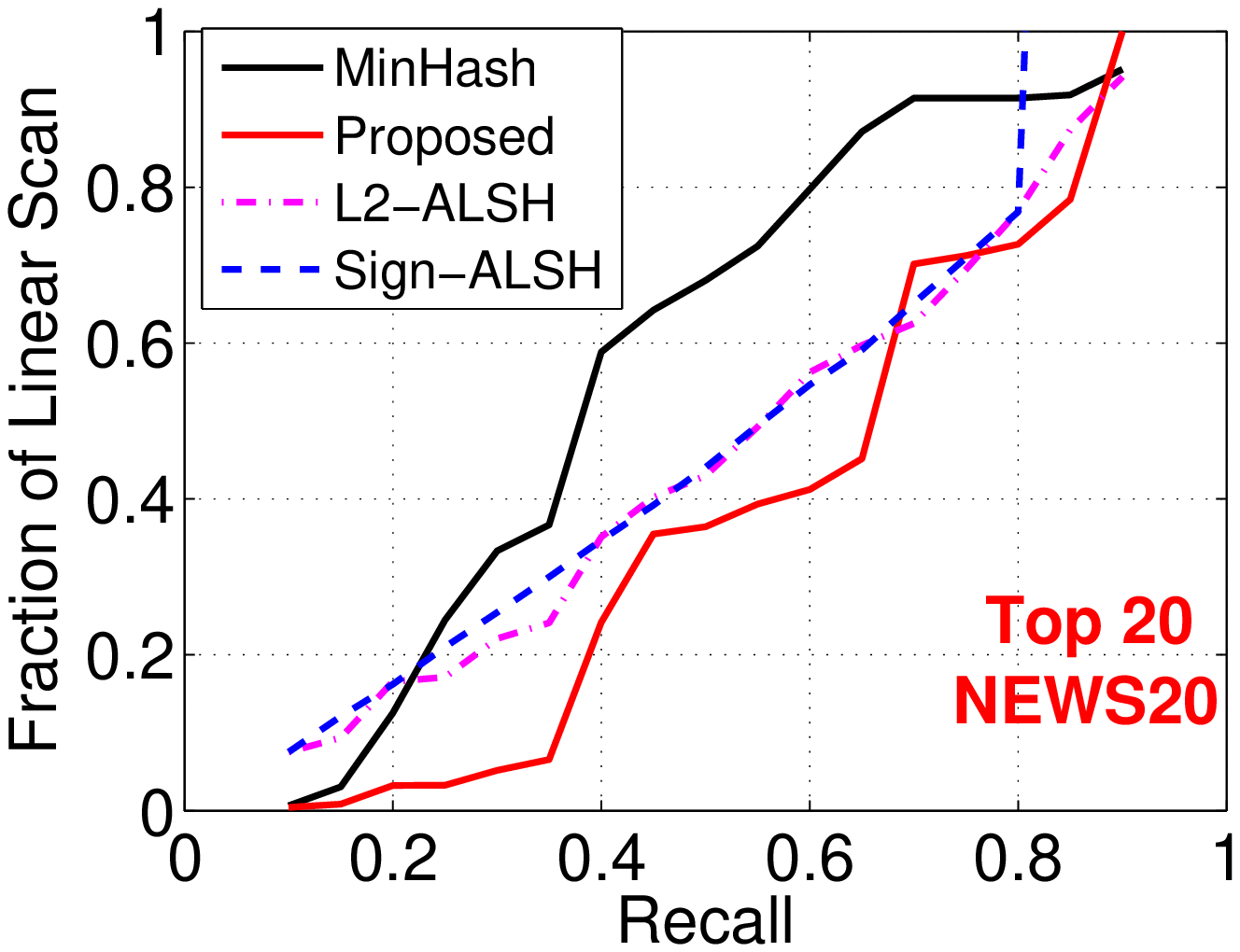}\hspace{-0.13in}
\includegraphics[width=1.8in]{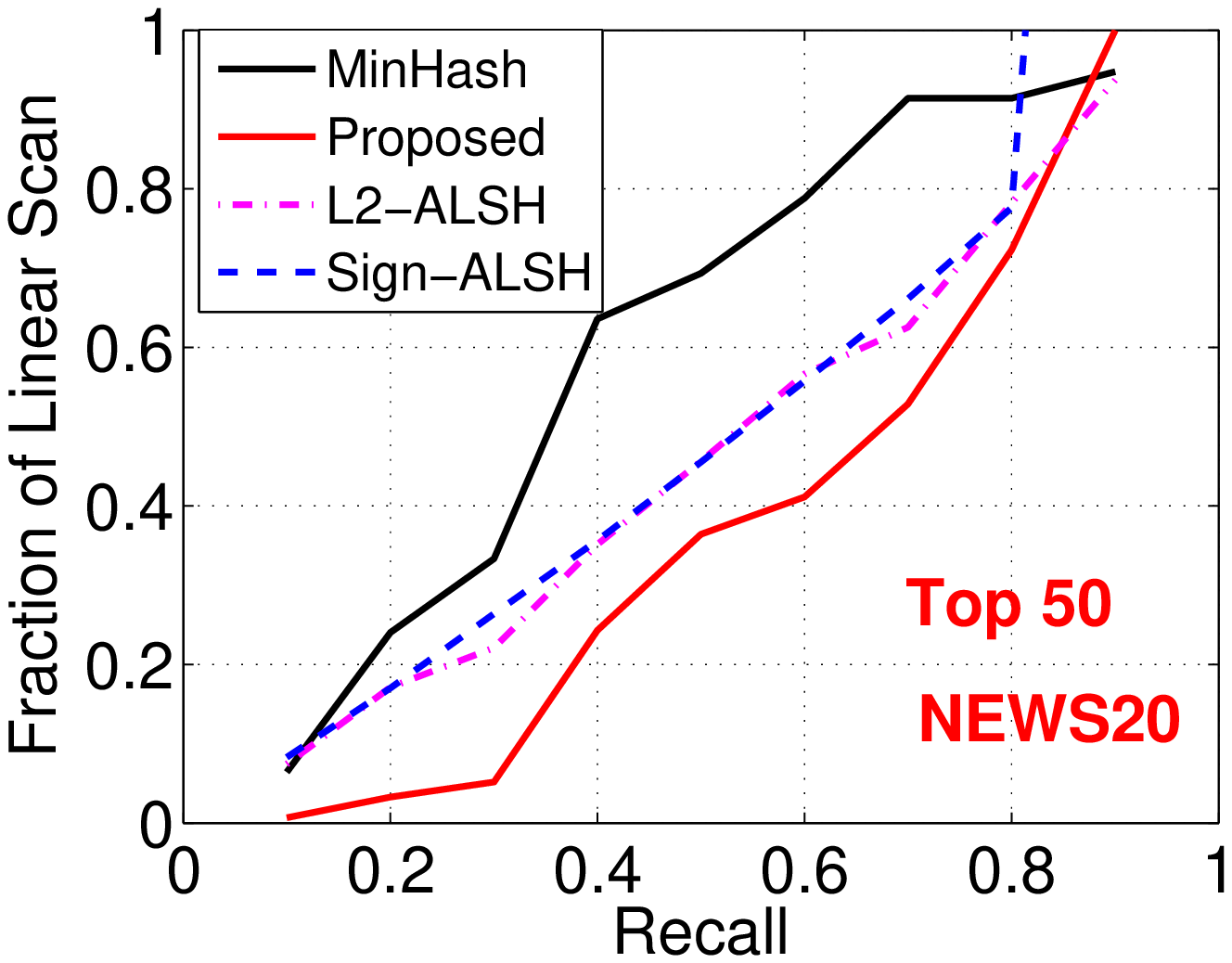}
}

\mbox{
\includegraphics[width=1.8in]{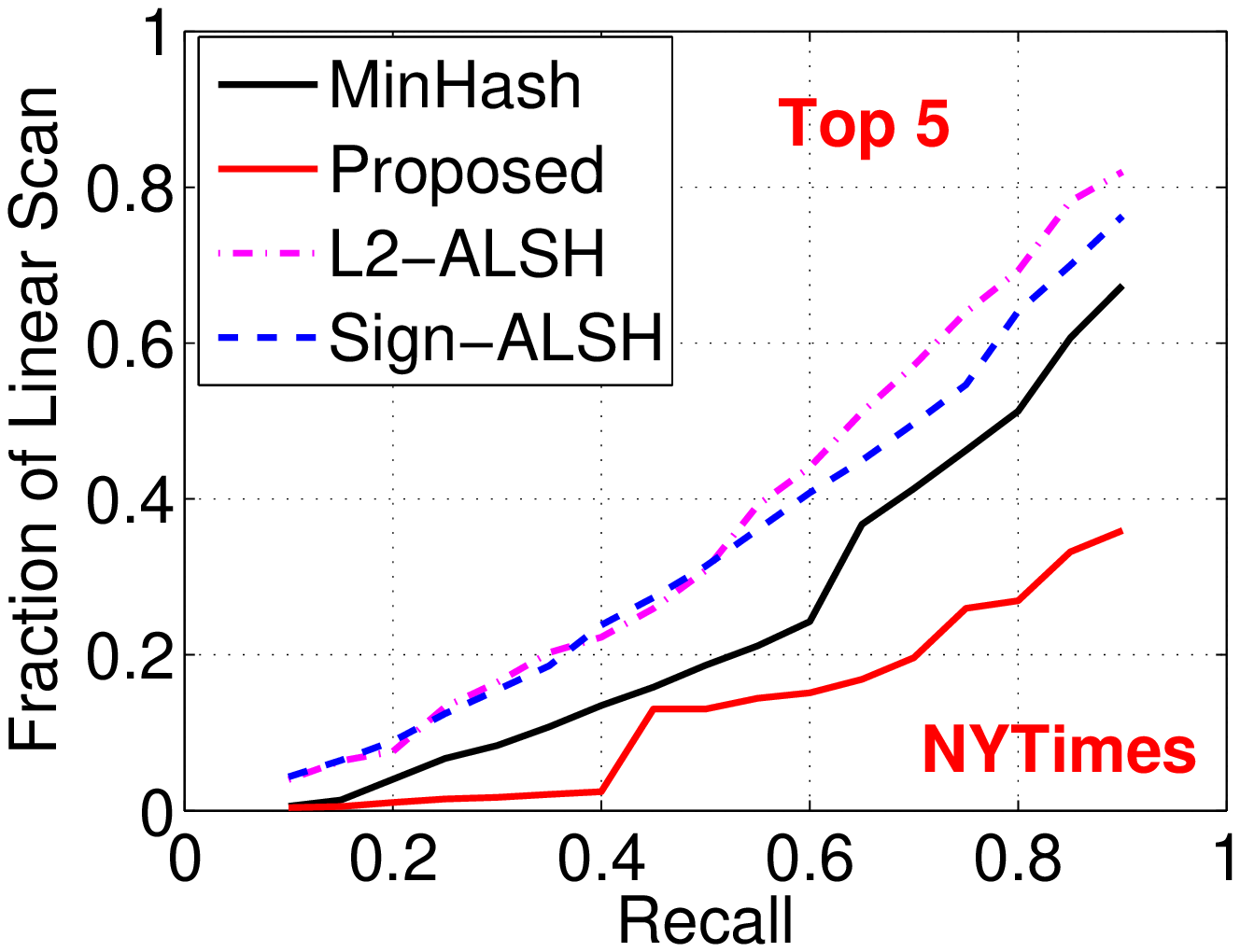}\hspace{-0.13in}
\includegraphics[width=1.8in]{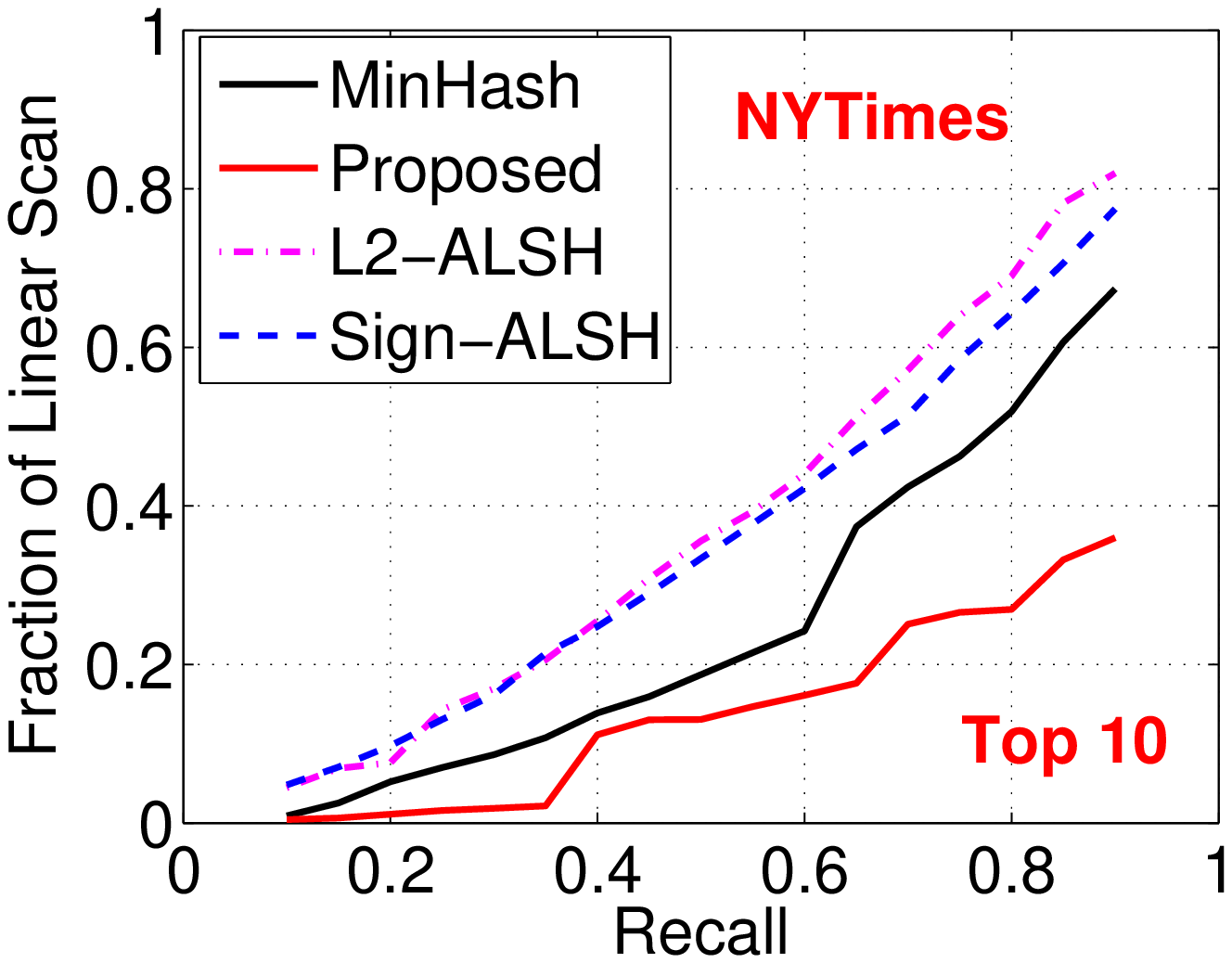}\hspace{-0.13in}
\includegraphics[width=1.8in]{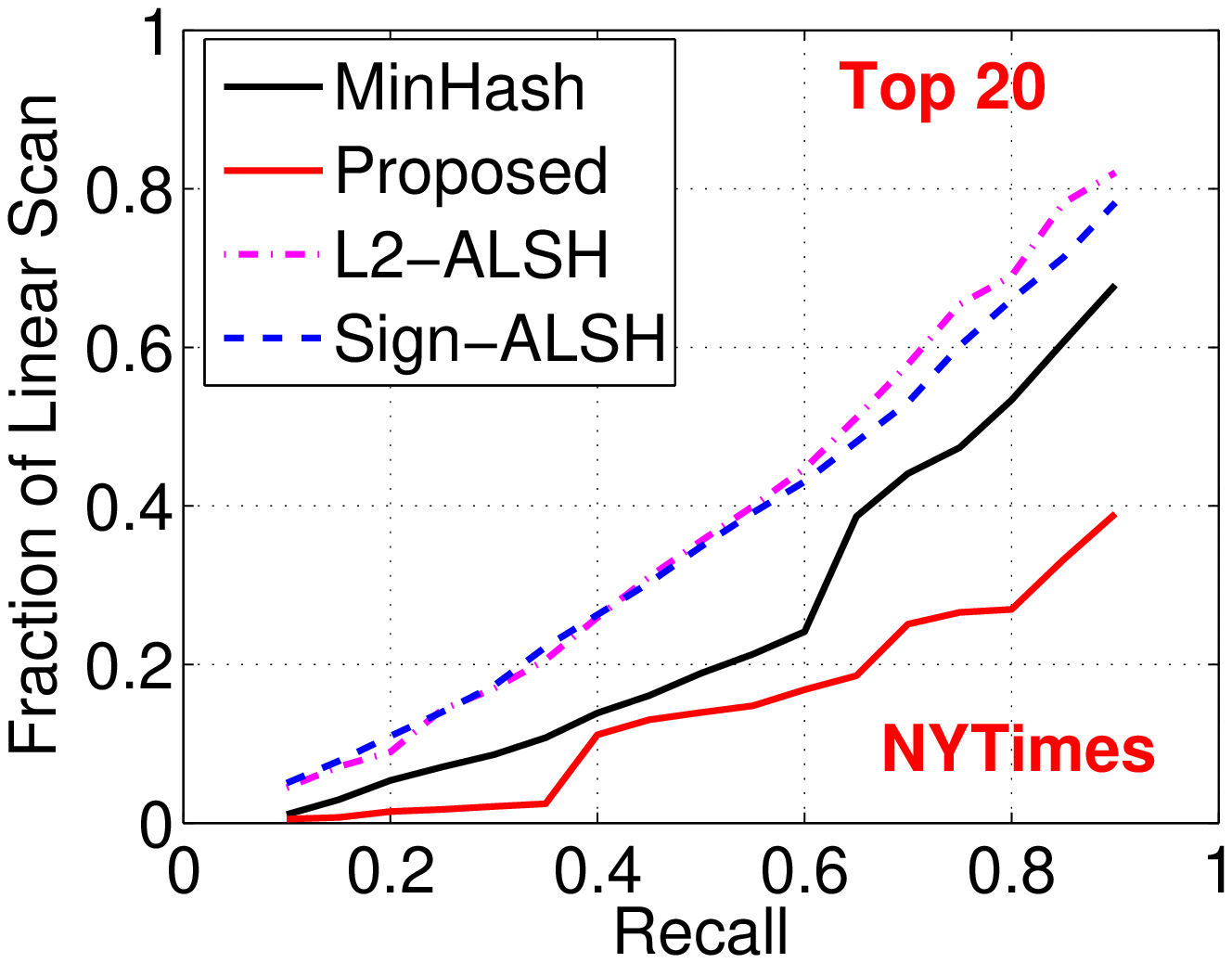}\hspace{-0.13in}
\includegraphics[width=1.8in]{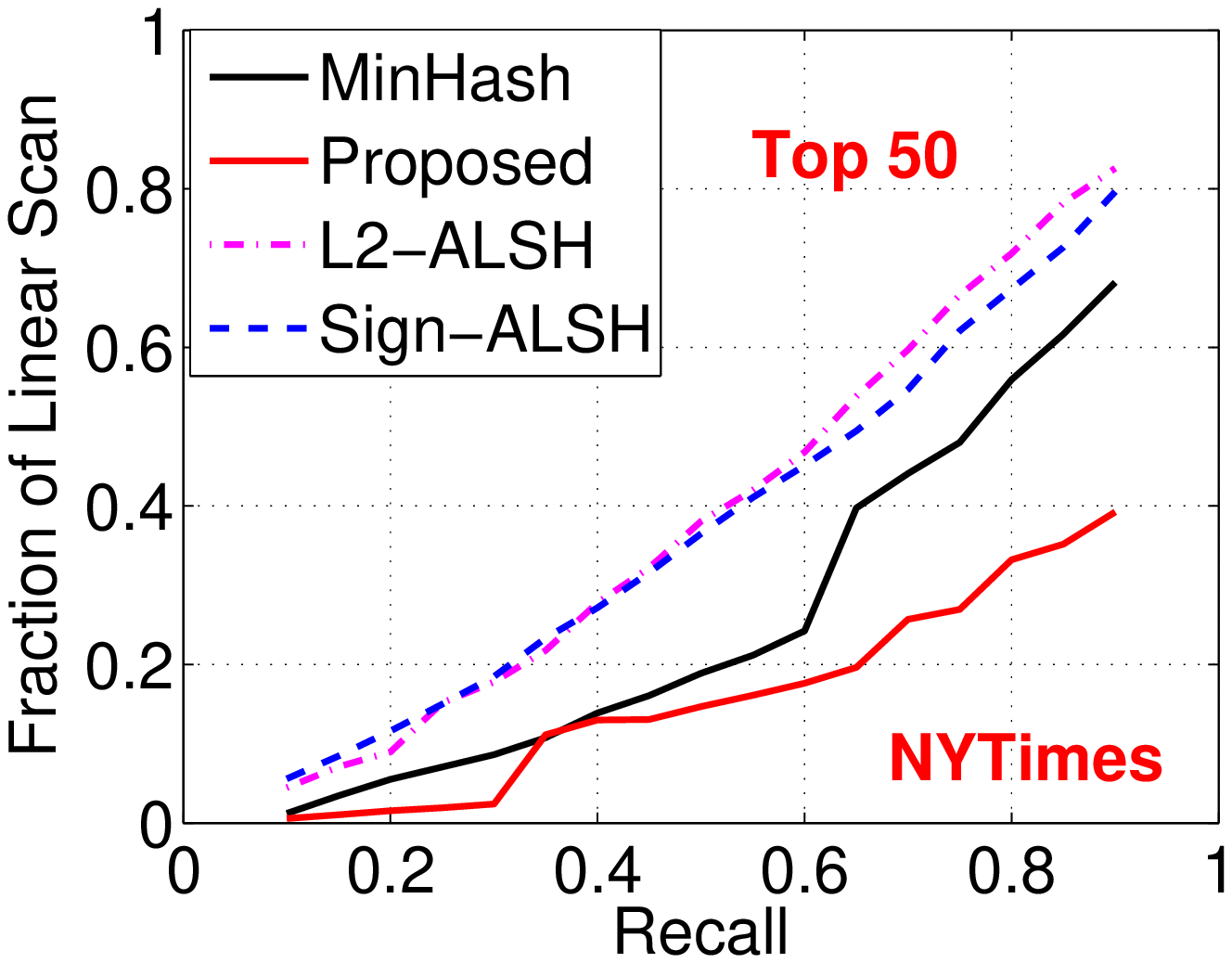}\hspace{-0.13in}
}
\end{center}
\vspace{-0.25in}
\caption{\textbf{LSH Bucketing Experiments.}\ \ Average number of  points retrieved per query (lower is better), relative to linear scan, evaluated by different hashing schemes at different recall levels, for top-5, \ top-10, \ top-20, \ top-50 nearest neighbors based on Jaccard containment (or equivalently inner products), on  four datasets. We show that results at the best $K$ and $L$ values chosen at every recall value, independently for each of the four hashing schemes.}\label{fig:topk}
\end{figure*}

The performance of a hashing based method varies with the variations in the similarity levels in the datasets. It can be seen that the proposed asymmetric minhash always retrieves much less number of points, and hence requires significantly less computations, compared to other hashing schemes at any recall level on all the four datasets. Asymmetric minhash consistently outperforms other hash functions irrespective of the operating point. The plots clearly establish the superiority of the proposed scheme for indexing Jaccard containment (or inner products).

 L2-ALSH and Sign-ALSH  perform better than traditional minhash on EP2006 and NEWS20 datasets while they are worse than plain minhash on NYTIMES and MNIST datasets. If we look at the statistics of the dataset from Table~\ref{tab_data}, NYTIMES and MNIST are precisely the datasets with less variations in the number of nonzeros and hence minhash performs better. In fact, for MNIST dataset with very small variations in the number of nonzeros, the performance of plain minhash  is very close  to the performance of asymmetric minhash. This is of course expected because there is negligible effect of penalization on the ordering. EP2006 and NEWS20 datasets have huge variations in their number of nonzeros and hence minhash performs very poorly on these datasets. What is exciting is that despite these variations in the nonzeros, asymmetric minhash always outperforms other ALSH for general  inner products.

The difference in the performance of plain minhash and asymmetric minhash clearly establishes the utility of our proposal which is simple and does not require any major modification over traditional minhash implementation. Given the fact that minhash is widely popular, we hope that our proposal will be adopted.

\section{Conclusion and Future Work}

Minwise hashing (minhash) is a widely popular indexing scheme in practice for similarity search. Minhash is originally designed for estimating set resemblance (i.e., normalized size of set intersections).  In many applications the performance of minhash is severely affected because minhash has a bias towards smaller sets. In this study, we propose asymmetric corrections (asymmetric minwise hashing, or MH-ALSH)  to minwise hashing that remove this often undesirable bias. Our corrections lead to a provably  superior algorithm for  retrieving binary inner products in the literature. Rigorous experimental evaluations on the task of retrieving maximum inner products clearly establish that the proposed approach can be significantly advantageous over the existing state-of-the-art hashing schemes in practice, when the desired similarity is the inner product (or containment) instead of the resemblance. Our proposed method requires only minimal modification of the original minwise hashing algorithm and should be straightforward to implement in practice. \\

\noindent \textbf{Future work}:\  One immediate future work would be {\em asymmetric consistent weighted sampling} for hashing weighted intersection:  $\sum_{i = 1} ^D \min \{x_i,y_i\}$, where  $x$ and $y$ are general real-valued  vectors. One proposal of the new  asymmetric transformation is the following:
\begin{align}
P(x) &= [x; M - \sum_{i=1}^D x_i ;0], \hspace{0.3in}
Q(x) = [x; 0; M - \sum_{i=1}^D x_i ],
\end{align}
where $M = \max_{x\in \mathcal{C}}\sum_{i} x_i$.   It is not difficult to show that the weighted Jaccard similarity between $P(x)$ and $Q(y)$ is monotonic in $\sum_{i = 1} ^D \min \{x_i,y_i\}$ as desired. At this point, we can use existing methods for consistent weighted sampling on the new data after asymmetric transformations~\cite{Report:Manasse_00,Proc:Ioffe_ICDM10,Report:Haeupler_arXiv14}.\\

Another potentially promising  topic for future work might be asymmetric minwise hashing for 3-way (or higher-order) similarities~\cite{Article:Li_Church_CL07,Proc:Shrivastava_NIPS13}

\end{document}